\documentclass[journal]{IEEEtran}
\ifCLASSINFOpdf
\else
\fi

\usepackage{tikz}
\usepackage{pgfplots}

\usepackage{amsmath,amsthm}
\usepackage{amssymb}
\usepackage{bbm}
\usepackage{amsfonts}
\usepackage{graphicx}
\usepackage{subfigure}
\graphicspath{{graphics/}}
\usepackage[ruled,vlined]{algorithm2e}
\usepackage{algorithmic}
\usepackage{tabularx,booktabs}
\usepackage{tikz}

\usepackage{xcolor}
\usepackage[colorlinks]{hyperref}
\hypersetup{colorlinks,breaklinks,linkcolor=blue,urlcolor=blue,anchorcolor=blue,citecolor=blue}
\usepackage{booktabs,caption}
\usepackage{multirow}
\usepackage{lscape}
\usepackage{epstopdf}
\usepackage{lineno}
\usepackage{subfigure}

\usepackage{microtype}
\usepackage{lineno}
\newtheorem{lemma}{Lemma}

\newtheorem{theorem}{Theorem}

\usepackage{lineno}

\usepackage{cite}

\begin{document}
%
% paper title
% Titles are generally capitalized except for words such as a, an, and, as,
% at, but, by, for, in, nor, of, on, or, the, to and up, which are usually
% not capitalized unless they are the first or last word of the title.
% Linebreaks \\ can be used within to get better formatting as desired.
% Do not put math or special symbols in the title.
\title{Low-Rank Autoregressive Tensor Completion for Spatiotemporal Traffic Data Imputation}
%
%
% author names and IEEE memberships
% note positions of commas and nonbreaking spaces ( ~ ) LaTeX will not break
% a structure at a ~ so this keeps an author's name from being broken across
% two lines.
% use \thanks{} to gain access to the first footnote area
% a separate \thanks must be used for each paragraph as LaTeX2e's \thanks
% was not built to handle multiple paragraphs
%

\author{Xinyu Chen,
        Mengying Lei,
        Nicolas Saunier,
        Lijun Sun{*}
\thanks{Xinyu Chen and Nicolas Saunier are with the Civil, Geological and Mining Engineering Department, Polytechnique Montreal, Montreal, QC H3T 1J4, Canada. E-mail: chenxy346@gmail.com (Xinyu Chen), nicolas.saunier@polymtl.ca (Nicolas Saunier).}% <-this % stops a space
\thanks{Mengying Lei and Lijun Sun are with the Department of Civil Engineering, McGill University, Montreal, QC H3A 0C3, Canada. E-mail: mengying.lei@mail.mcgill.ca (Mengying Lei), lijun.sun@mcgill.ca (Lijun Sun).}% <-this % stops a space
\thanks{* Corresponding author. Address: 492-817 Sherbrooke Street West, Macdonald Engineering Building, Montreal, Quebec H3A 0C3, Canada}
\thanks{Manuscript received xx; revised xx.}}

\maketitle

% As a general rule, do not put math, special symbols or citations
% in the abstract or keywords.
\begin{abstract}
Spatiotemporal traffic time series (e.g., traffic volume/speed) collected from sensing systems are often incomplete with considerable corruption and large amounts of missing values,  preventing users from harnessing the full power of the data. Missing data imputation has been a long-standing research topic and critical application for real-world intelligent transportation systems. A widely applied imputation method is low-rank matrix/tensor completion; however, the low-rank assumption only preserves the global structure while ignores the strong local consistency in spatiotemporal data. In this paper, we propose a low-rank autoregressive tensor completion (LATC) framework by introducing \textit{temporal variation} as a new regularization term into the completion of a third-order (sensor $\times$ time of day $\times$ day) tensor. The third-order tensor structure allows us to better capture the global consistency of traffic data, such as the inherent seasonality and day-to-day similarity. To achieve local consistency, we design the temporal variation by imposing an AR($p$) model for each time series with coefficients as learnable parameters. Different from previous spatial and temporal regularization schemes, the minimization of temporal variation can better characterize temporal generative mechanisms beyond local smoothness, allowing us to deal with more challenging scenarios such ``blackout" missing. To solve the optimization problem in LATC, we introduce an alternating minimization scheme that estimates the low-rank tensor and autoregressive coefficients iteratively. We conduct extensive numerical experiments on several real-world traffic data sets, and our results demonstrate the effectiveness of LATC in diverse missing scenarios.
\end{abstract}

% Note that keywords are not normally used for peerreview papers.

% For peer review papers, you can put extra information on the cover
% page as needed:
% \ifCLASSOPTIONpeerreview
% \begin{center} \bfseries EDICS Category: 3-BBND \end{center}
% \fi
%
% For peerreview papers, this IEEEtran command inserts a page break and
% creates the second title. It will be ignored for other modes.
\IEEEpeerreviewmaketitle

\section{Introduction}

% {\color{red} task, motivation, existing smoothing work, contribution of our work.} \\
% {\color{red} L1 regularization for A estimation.}

Spatiotemporal traffic data collected from various sensing systems (e.g. loop detectors and floating cars) serve as the foundation to a wide range of applications and decision-making processes in intelligent transportation systems. The emerging ``big''  data is often large-scale, high-dimensional, and incomplete, posing new challenges to modeling spatiotemporal traffic data. Missing data imputation is one of the most important research questions in spatiotemporal data analysis, since accurate and reliable imputation can help various downstream applications such as traffic forecasting and traffic control/management.

The key to missing data imputation is to efficiently characterize and leverage the complex dependencies and correlations across both spatial and temporal dimensions \cite{bahadori2014fast}. Different from point-referenced systems, traffic state data (e.g., speed and flow) is individual sensor-based with a fixed temporal resolution. This allows us to summarize spatiotemporal traffic state data in the format of a matrix (e.g., sensor $\times$ time) or a tensor (e.g., sensor $\times$ time of day $\times$ day) \cite{chen2019abayesian}, and low-rank matrix/tensor completion becomes a natural solution to solve the imputation problem. Over the past decade, extensive effort has been made on developing low-rank models through principle component analysis, matrix/tensor factorization (with predefined rank) and nuclear norm minimization (see e.g., \cite{li2009dynammo, chen2019abayesian, chen2020anonconvex}). However, the default low-rank structure (e.g., nuclear norm) purely relies on the algebraic property of the data, which is invariant to permutation in the spatial and temporal dimensions. In other words, with the low-rank assumption alone, we essentially overlook the strong ``local" spatial and temporal consistency in the data. For instance, we expect traffic flow data collected in a short period to be similar and adjacent sensors to show similar patterns. To this end, some recent studies have tried to encode such ``local" consistency by introducing total/quadratic variation and graph regularization as a ``smoothness'' prior into low-rank factorization models \cite{xiong2010temporal,bahadori2014fast,rao2015collaborative,yokota2016smooth} and imposing time series dynamics on the temporal latent factor in the factorization framework \cite{yu2016temporal, sen2019think, chen2021bayesian}. However, these studies essentially adopt a bilinear/multilinear factorization model, which requires a predefined rank as a hyperparameter.

In this paper, we propose a low-rank autoregressive tensor completion (LATC) framework to impute missing values in spatiotemporal traffic data. For each completed time series, we define temporal variation as the accumulated sum of autoregressive errors. To model the low-rankness property, we use truncated nuclear norm as an effective approximation to avoid the rank determination problem in factorization models. The final objective function of LATC consists of two components, i.e., the truncated nuclear of the completed tensor and the temporal variation defined on the unfolded time series matrix. The combination allows us to effectively characterize both global patterns and local consistency in spatiotemporal traffic data. The overall contribution of this work is threefold:
% Our fundamental assumption is that the original traffic time series matrix can be transformed to a third-order tensor based on the most important seasonality. Then, we build a new autoregressive norm on the multivariate time series matrix to characterize temporal patterns, which can also characterize generative mechanisms beyond temporal smoothness.
\begin{enumerate}
    \item[1)] We integrate the autoregressive time series process into a low-rank tensor completion model to capture both global and local trends in spatiotemporal traffic data. By minimizing the truncated nuclear norm of the third-order (sensor$\times$time of day$\times$day) tensor, we can better characterize day-to-day similarity, which is a unique property of traffic time series data \cite{li2015trend}.
    \item[2)] We develop an alternating learning algorithm to update tensor and coefficient matrix separately. The tensor is updated via ADMM, and the coefficient matrix is updated by least squares with closed-form solution.
    \item[3)] We conduct extensive numerical experiments on four traffic data sets. Imputation results show the superiority and advantage of LATC over recent state-of-the-art models.
\end{enumerate}

The remainder of this paper is organized as follows. We introduce related work and notations in Section~\ref{related_work} and Section~\ref{preliminaries}, respectively. Section~\ref{mothodology} introduces in detail the proposed LATC model. In Section~\ref{experiments}, we conduct extensive experiments on some traffic data sets and make comparison with some baseline models. Finally, we summarize the study in Section~\ref{conclusion}.

\section{Related Work}\label{related_work}

There are two types of low-rank models to solve the spatiotemporal missing data imputation problem.

\noindent\textbf{Temporal matrix factorization}. Factorization models approximate the complete spatiotemporal matrix/tensor using bilinear/multilinear factorization models with a predefined rank parameter. To encode temporal consistency, recent studies have introduced local smoothness and time series dynamics to regularize the temporal factor (see e.g., \cite{xiong2010temporal, yu2016temporal, sen2019think, chen2021bayesian}). The introduction of generative mechanism (e.g., autoregressive model) not only offers better interpolation/imputation accuracy, but also enable the factorization models to perform forecasting. However, a major limitation of these models is that they often require careful tuning and selection of the rank parameter.

% \noindent\textbf{Hankel embedding}. Singular spectrum analysis (SSA) and Hankel structured low-rank completion are powerful approaches for time series analysis \cite{golyandina2001analysis, markovskylow}. They are model-free approaches
% to detect spectral patterns at different scales in time series data. Specifically, SSA applies singular value decomposition (SVD) on the Hankel matrix obtained from the original univariate time series, and then uses those principal components to analyze the time series. The default SSA model is for univariate time series but it can be easily extended to the multivariate time series. This approach can accomplish both missing data imputation and prediction tasks by performing low-rank completion on the Hankel matrix\cite{chen2014robust, gillard2018structured, agarwal2018model, yokota2018missing, zhang2019correction}. However, such model is computationally expensive to work with because the Hankel matrices/tensors would be very large in some cases.

\noindent\textbf{Tensor representation}. Another approach is to fold a time series matrix into a third-order tensor (sensor $\times$ time of day $\times$ day) by introducing an additional ``day'' dimension (e.g., \cite{tan2016short, li2019tensor, chen2020anonconvex}). This is a particular case for traffic data given the clear day-to-day similarity, but many real-world time series data resulted from human behavior/activities (e.g., energy/electricity consumption) also exhibit similar patterns. It is expected that the third-order representation captures more information, given that the multivariate time series matrix is in fact one of the unfoldings of the third-order tensor. As a result, the tensor structure not only preserves the dependencies among sensors but also provides an alternative to capture both local and global temporal patterns (e.g., traffic speed data at 9:00 am on Monday might be similar to that of 9:00 am on Tuesday). These tensor-based models have shown superior performance over matrix-based models in missing data imputation tasks.

\section{Notations}\label{preliminaries}

Throughout this work, we use boldface uppercase letters to denote matrices, e.g., $\boldsymbol{X}\in\mathbb{R}^{M\times N}$, boldface lowercase letters to denote vectors, e.g., $\boldsymbol{x}\in\mathbb{R}^{M}$, and lowercase letters to denote scalars, e.g., $x$. Given a matrix $\boldsymbol{X}\in\mathbb{R}^{M\times N}$, we denote the $(m,n)$th entry in $\boldsymbol{X}$ by $x_{m,n}$, and use $\boldsymbol{x}_{m,[t+1:]}\in\mathbb{R}^{(N-t)}$ to denote the sub-vector that consists of the last $N-t$ entries of $\boldsymbol{x}_{m}\in\mathbb{R}^{N}$. The Frobenius norm of $\boldsymbol{X}$ is defined as $\|\boldsymbol{X}\|_{F}=\sqrt{\sum_{m,n}x_{m,n}^{2}}$, and the $\ell_2$-norm of $\boldsymbol{x}$ is defined as $\|\boldsymbol{x}\|_{2}=\sqrt{\sum_{m}x_{m}^{2}}$. We denote a third-order tensor by $\boldsymbol{\mathcal{X}}\in\mathbb{R}^{M\times I\times J}$ and the $k$th-mode ($k=1,2,3$) unfolding of $\boldsymbol{\mathcal{X}}$ by $\boldsymbol{\mathcal{X}}_{(k)}$ \cite{kolda2009tensor}. Correspondingly, the folding operator $\operatorname{fold}_{k}(\cdot)$ converts a matrix to a third-order tensor in the $k$th-mode. Thus, we have  $\operatorname{fold}_{k}(\boldsymbol{\mathcal{X}}_{(k)})=\boldsymbol{\mathcal{X}}$ for any tensor $\boldsymbol{\mathcal{X}}$. For $\boldsymbol{\mathcal{X}}\in\mathbb{R}^{M\times I\times J}$, its Frobenius norm is defined as $\|\boldsymbol{\mathcal{X}}\|_{F}=\sqrt{\sum_{m,i,j}x_{m,i,j}^{2}}$ and its inner product with another tensor is given by $\left\langle\boldsymbol{\mathcal{X}},\boldsymbol{\mathcal{Y}}\right\rangle=\sum_{m,i,j}x_{m,i,j}y_{m,i,j}$ where $\boldsymbol{\mathcal{Y}}$ and $\boldsymbol{\mathcal{X}}$ are of the same size.

\section{Methodology}\label{mothodology}

\subsection{Tensorization for Global Consistency}

We denote the true spatiotemporal traffic data collected from $M$ sensors over $J$ days by $\boldsymbol{Y}$, whose columns correspond to time points and rows correspond to sensors:
\begin{equation}
    \boldsymbol{Y}=\left[\begin{array}{cccc}
    \mid & \mid & & \mid \\
    \boldsymbol{y}_{1} & \boldsymbol{y}_{2} & \cdots & \boldsymbol{y}_{IJ} \\
    \mid & \mid & & \mid
    \end{array}\right]\in\mathbb{R}^{M\times (IJ)},
\end{equation}
where $I$ is the number of time points per day. The observed/incomplete matrix can be written as $\mathcal{P}_{\Omega}(\boldsymbol{Y})$ with observed entries on the support $\Omega$:
\begin{equation} \notag
    [\mathcal{P}_{\Omega}(\boldsymbol{Y})]_{m,n}=\left\{\begin{array}{ll}
    y_{m,n},     & \text{if $(m,n)\in\Omega$,}  \\
    0,     & \text{otherwise}, \\
    \end{array}\right.
\end{equation}
where $m=1,\ldots,M$ and $n=1,\ldots,IJ$.

We next introduce the forward tensorization operator $\mathcal{Q}(\cdot)$ that converts the multivariate time series matrix into a third-order tensor. Temporal dimension of traffic time series is divided into two dimensions, i.e., time of day and day. Formally, a third-order tensor can be generated by the forward tensorization operator as $\boldsymbol{\mathcal{X}}=\mathcal{Q}(\boldsymbol{Y})\in\mathbb{R}^{M\times I\times J}$. Conversely, the resulted tensor can also be converted into the original matrix by $\boldsymbol{Y}=\mathcal{Q}^{-1}(\boldsymbol{\mathcal{X}})\in\mathbb{R}^{M\times (IJ)}$ where $\mathcal{Q}^{-1}(\cdot)$ denotes the inverse operator of $\mathcal{Q}(\cdot)$.

% Until now, we can build a autoregressive time series model on the matrix data. In the meanwhile, tensor representation provides an option for uncovering low-rank patterns.

The tensorization step transforms matrix-based imputation problem to a low-rank tensor completion problem. Global consistency can be achieved by minimizing tensor rank. In practice, tensor rank is often approximated using sum of nuclear norms $\|\boldsymbol{\mathcal{X}}\|_{*}$ \cite{liu2013tensor} or truncated nuclear norms $\|\boldsymbol{\mathcal{X}}\|_{r,*}$ \cite{chen2020anonconvex}, where $r$ is a truncation parameter (see section~\ref{sec:latc}). Our motivation for doing so is that the spatiotemporal traffic data can be characterized by both long-term global trends and short-term local trends. The long-term trends refer to certain periodic, seasonal, and cyclical patterns. Traffic flow data over 24 hours on a typical weekday often shows a systematic ``M" shape resulted from travelers' behavioral rhythms, with two peaks during morning and evening rush hours \cite{lai2018modeling}. The pattern also exists at the weekly level with substantial differences from weekdays to weekends. The short-term trends capture certain temporary volatility/perturbation that deviates from the global patterns (e.g., due to incident or special event). The short-term trends seem to be more ``random", but they are common and ubiquitous in reality. LATC leverages both global and local patterns by
using matrix and tensor simultaneously.

\subsection{Temporal Variation for Local Consistency}

We define temporal variation of a time series matrix $\boldsymbol{Z}$ given a coefficient matrix $\boldsymbol{A}\in\mathbb{R}^{M\times d}$ and a time lag set $\mathcal{H}=\{h_1,\ldots,h_d\}$ as
\begin{equation}
    \|\boldsymbol{Z}\|_{\boldsymbol{A},\mathcal{H}}=\sum_{m,t}(z_{m,t}-\sum_{i}a_{m,i}z_{m,t-h_i})^{2}.
    \label{auto_norm}
\end{equation}

As can be seen, $\|\boldsymbol{Z}\|_{\boldsymbol{A},\mathcal{H}}$ quantifies the total squared error when fitting each individual time series $\boldsymbol{z}_m$ with an autoregressive model with coefficient $\boldsymbol{a}_m$. Given an estimated $\boldsymbol{A}$, minimizing the temporal variation will encourage the time series data $\boldsymbol{Z}$ to show stronger temporal consistency. In other words, the multivariate time series matrix $\boldsymbol{Z}$ will be better explained by a series of  autoregressive models parameterized by $\boldsymbol{A}$. It should be noted that both $\boldsymbol{Z}$ and $\boldsymbol{A}$ are variables in the proposed temporal variation term.

\subsection{Low-rank Autoregressive Tensor Completion (LATC)} \label{sec:latc}

The ensure both global consistency and local consistency, we propose LATC as the following optimization model
\begin{equation}
    \begin{aligned}
    \min _{\boldsymbol{\mathcal{X}},\boldsymbol{Z},\boldsymbol{A}}~&\|\boldsymbol{\mathcal{X}}\|_{r,*}+\frac{\lambda}{2}\|\boldsymbol{Z}\|_{\boldsymbol{A},\mathcal{H}} \\ \text { s.t.}~&\left\{\begin{array}{l}\boldsymbol{\mathcal{X}}=\mathcal{Q}\left(\boldsymbol{Z}\right), \\ \mathcal{P}_{\Omega}(\boldsymbol{Z})=\mathcal{P}_{\Omega}(\boldsymbol{Y}), \\ \end{array}\right. \\
    \end{aligned}
    \label{lrtc_ar}
\end{equation}
where $\boldsymbol{Y}\in\mathbb{R}^{M\times (IJ)}$ is the partially observed time series matrix. $r\in\mathbb{N}_{+}$ is the truncation which satisfies $r<\min\{M,I,J\}$.

The formulation of LATC ensures both global consistency and local consistency by combining truncated nuclear norm minimization with temporal variation minimization. The weight parameter $\lambda$ in the objective function controls the trade-off between truncated nuclear norm and temporal variation. Fig.~\ref{framework} shows that $\boldsymbol{Y}$ can be reconstructed with both low-rank properties and time series dynamics because the constraint in ~\eqref{lrtc_ar}, i.e., $\boldsymbol{\mathcal{X}}=\mathcal{Q}(\boldsymbol{Z})$, is closely related to the partially observed matrix $\boldsymbol{Y}$.

\begin{figure*}[!ht]
\centering
  \includegraphics[width=0.75\textwidth]{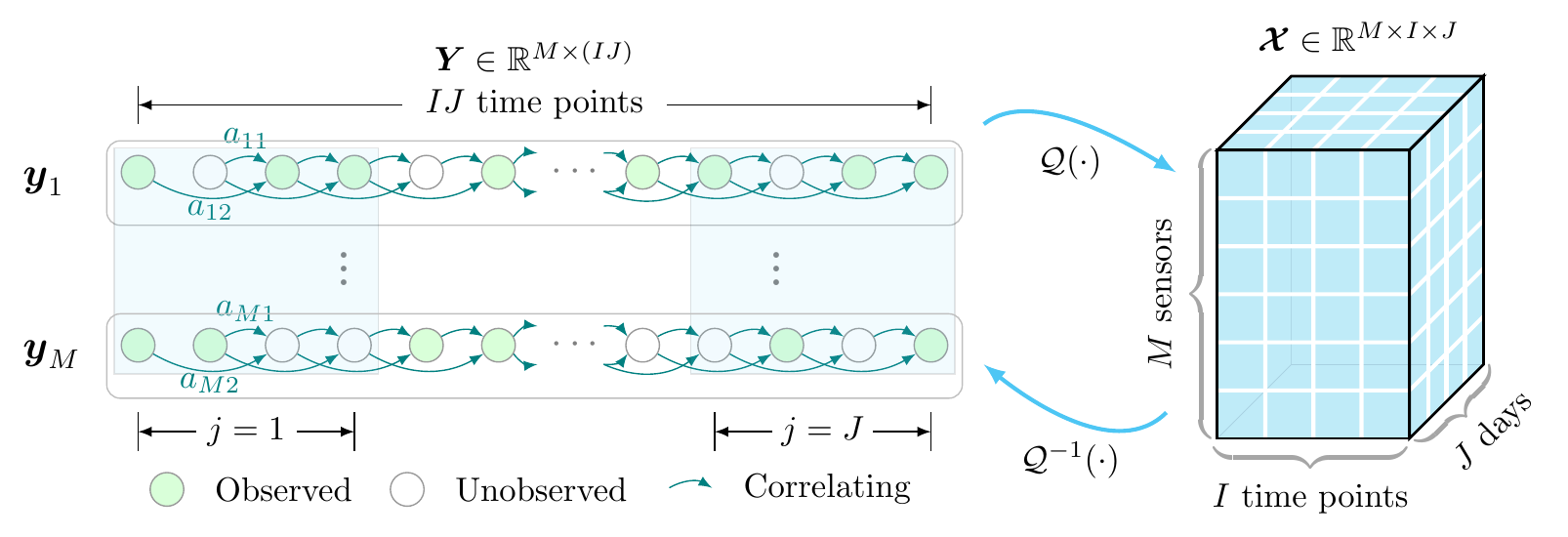}
\caption{Illustration of the proposed LATC framework for spatiotemporal traffic data imputation with time lags $\mathcal{H}=\{1,2\}$. Each time series $\boldsymbol{y}_m,\forall m\in\{1,2,\ldots,M\}$ is modeled by the autoregressive coefficients $\{a_{m1},a_{m2}\}$.}
\label{framework}
\end{figure*}

Most nuclear norm-based tensor completion models employ the Alternating Direction Method of Multipliers (ADMM) algorithm to solve the optimization problem. However, due to the introduction of autoregression coefficient matrix, we can no longer apply the default ADMM algorithm to solve the optimization problem~\eqref{lrtc_ar}. Here we consider applying an alternating minimization scheme by separating the original optimization into two subproblems. Starting with some given initial values $(\boldsymbol{\mathcal{X}}^{0},\boldsymbol{Z}^{0},\boldsymbol{A}^{0})$, we can update $\{(\boldsymbol{\mathcal{X}}^{\ell},\boldsymbol{Z}^{\ell},\boldsymbol{A}^{\ell})\}_{\ell\in\mathbb{N}}$ by solving the two subproblems in an iterative manner. In the implementation, we first fix  $\boldsymbol{A}^{\ell}$ and solve the following problem to update the variables $\boldsymbol{\mathcal{X}}^{\ell+1}$ and $\boldsymbol{Z}^{\ell+1}$:
\begin{equation}\label{admm_prob}
    \begin{aligned}
    \boldsymbol{\mathcal{X}}^{\ell+1},\boldsymbol{Z}^{\ell+1}&:=\operatorname{arg}\min _{\boldsymbol{\mathcal{X}},\boldsymbol{Z}}~\|\boldsymbol{\mathcal{X}}\|_{r,*}+\frac{\lambda}{2}\|\boldsymbol{Z}\|_{\boldsymbol{A}^{\ell},\mathcal{H}} \\ \text { s.t.}~&\left\{\begin{array}{l}\boldsymbol{\mathcal{X}}=\mathcal{Q}(\boldsymbol{Z}), \\ \mathcal{P}_{\Omega}(\boldsymbol{Z})=\mathcal{P}_{\Omega}(\boldsymbol{Y}). \\ \end{array}\right. \\
    \end{aligned}
\end{equation}
where $\ell$ denotes the count of iteration in the alternating minimization scheme. Then, we fix $\boldsymbol{Z}^{\ell+1}$ and solve the following least square problem to estimate the coefficient matrix $\boldsymbol{A}^{\ell+1}$:
\begin{equation}
    \boldsymbol{A}^{\ell+1}:=\operatorname{arg}\min _{\boldsymbol{A}}~\|\boldsymbol{Z}^{\ell+1}\|_{\boldsymbol{A},\mathcal{H}}.
    \label{ar_minimize}
\end{equation}

When $\boldsymbol{A}^{\ell}$ is fixed, the subproblem in Eq.~\eqref{admm_prob} becomes a general low-rank tensor problem, can it can be solved using ADMM in a similar way as in \cite{liu2013tensor} and \cite{hu2013fast}. The augmented Lagrangian function of the optimization in Eq.~\eqref{admm_prob} can be written as
% \begin{equation}\label{augmented_lagrangian_func}
%     \mathcal{L}_{\rho}(\boldsymbol{\mathcal{X}},\boldsymbol{Z},\boldsymbol{\mathcal{T}})=\|\boldsymbol{\mathcal{X}}\|_{r,*}+\frac{\lambda}{2}\|\boldsymbol{Z}\|_{\boldsymbol{A}^{\ell},\mathcal{H}}+\frac{\rho}{2}\|\boldsymbol{\mathcal{X}}-\mathcal{Q}(\boldsymbol{Z})\|_{F}^{2}+\big\langle\boldsymbol{\mathcal{X}}-\mathcal{Q}(\boldsymbol{Z}),\boldsymbol{\mathcal{T}}\big\rangle,
% \end{equation}
\begin{equation}\label{augmented_lagrangian_func}
\begin{split}
    \mathcal{L}(\boldsymbol{\mathcal{X}},\boldsymbol{Z},\boldsymbol{A}^{\ell},\boldsymbol{\mathcal{T}})=&\|\boldsymbol{\mathcal{X}}\|_{r,*}+\frac{\lambda}{2}\|\boldsymbol{Z}\|_{\boldsymbol{A}^{\ell},\mathcal{H}}\\
    &+\frac{\rho}{2}\|\boldsymbol{\mathcal{X}}-\mathcal{Q}(\boldsymbol{Z})\|_{F}^{2}\\
    &+\big\langle\boldsymbol{\mathcal{X}}-\mathcal{Q}(\boldsymbol{Z}),\boldsymbol{\mathcal{T}}\big\rangle,
\end{split}
\end{equation}
where $\rho$ is the learning rate of ADMM, and $\boldsymbol{\mathcal{T}}\in\mathbb{R}^{M\times I\times J}$ is the dual variable. In particular, we keep  $\mathcal{P}_{\Omega}(\boldsymbol{Z})=\mathcal{P}_{\Omega}(\boldsymbol{Y})$ as a fixed constraint to maintain observation consistency. According to the augmented Lagrangian function, ADMM can transform the problem in Eq.~\eqref{admm_prob} into the following subproblems in an iterative manner:
\begin{align}
    \boldsymbol{\mathcal{X}}^{\ell+1,k+1}:&=\operatorname{arg}\min_{\boldsymbol{\mathcal{X}}}~\mathcal{L}(\boldsymbol{\mathcal{X}},\boldsymbol{Z}^{\ell+1,k},\boldsymbol{A}^{\ell},\boldsymbol{\mathcal{T}}^{\ell+1,k}),\label{sub1} \\
    \boldsymbol{Z}^{\ell+1,k+1}:&=\operatorname{arg}\min_{\boldsymbol{Z}}~\mathcal{L}(\boldsymbol{\mathcal{X}}^{\ell+1,k+1},\boldsymbol{Z},\boldsymbol{A}^{\ell},\boldsymbol{\mathcal{T}}^{\ell+1,k}),\label{sub2} \\
    \boldsymbol{\mathcal{T}}^{\ell+1,k+1}:&=\boldsymbol{\mathcal{T}}^{\ell+1,k}+\rho(\boldsymbol{\mathcal{X}}^{\ell+1,k+1}-\mathcal{Q}(\boldsymbol{Z}^{\ell+1,k+1})),\label{sub3}
\end{align}
where $k$ denotes the count of iteration in the ADMM. In the following, we discuss the detailed solutions to Eqs.~\eqref{sub1} and \eqref{sub2}.

\subsubsection{Update Variable $\boldsymbol{\mathcal{X}}$}

The optimization over $\boldsymbol{\mathcal{X}}$ is a truncated nuclear norm minimization problem. Truncated nuclear norm of any given tensor is the weighted sum of truncated nuclear norm on the unfolding matrices of the tensor, which takes the form:
\begin{equation}
\|\boldsymbol{\mathcal{X}}\|_{r,*}=\sum_{p=1}^{3}\alpha_p\|\boldsymbol{\mathcal{X}}_{(p)}\|_{r,*}
\end{equation}
for tensor $\boldsymbol{\mathcal{X}}\in\mathbb{R}^{M\times I\times J}$ with $\sum_{p=1}^{3}\alpha_{p}=1$. For the minimization of truncated nuclear norm on tensor, the above formula is not in its appropriate form because unfolding a tensor in different modes cannot guarantee the dependencies of variables \cite{liu2013tensor}. Therefore, we introduce $\boldsymbol{\mathcal{X}}_{1},\boldsymbol{\mathcal{X}}_{2},\boldsymbol{\mathcal{X}}_{3}$ and they correspond to the unfoldings of $\boldsymbol{\mathcal{X}}$. Accordingly, it is possible to obtain the closed-form solution for each $\boldsymbol{\mathcal{X}}_{p}$:
\begin{equation}\label{computeX}
\small
\begin{aligned}
    \boldsymbol{\mathcal{X}}_{p}:=&\operatorname{arg}\min_{\boldsymbol{\mathcal{X}}}~\alpha_{p}\|\boldsymbol{\mathcal{X}}_{(p)}\|_{r,*}+\frac{\rho}{2}\left\|\mathcal{Q}^{-1}(\boldsymbol{\mathcal{X}})-\boldsymbol{Z}^{\ell+1,k}\right\|_{F}^{2} \\
    &+\big\langle\mathcal{Q}^{-1}(\boldsymbol{\mathcal{X}})-\boldsymbol{Z}^{\ell+1,k},\mathcal{Q}^{-1}(\boldsymbol{\mathcal{T}}^{\ell+1,k})\big\rangle \\
    =&\operatorname{arg}\min_{\boldsymbol{\mathcal{X}}}~\alpha_{p}\|\boldsymbol{\mathcal{X}}_{(p)}\|_{r,*}\\
    &+\frac{\rho}{2}\left\|\boldsymbol{\mathcal{X}}-\left(\mathcal{Q}(\boldsymbol{Z}^{\ell+1,k})-\boldsymbol{\mathcal{T}}^{\ell+1,k}/\rho\right)\right\|_{F}^{2} \\
    =&\operatorname{fold}_{p}\left(\mathcal{D}_{r,\alpha_p/\rho}\left(\mathcal{Q}(\boldsymbol{Z}^{\ell+1,k})_{(p)}-\boldsymbol{\mathcal{T}}_{(p)}^{\ell+1,k}/\rho\right)\right), \\
\end{aligned}
\end{equation}
where $\mathcal{D}_{\cdot}(\cdot)$ denotes the generalized singular value thresholding that associated with truncated nuclear norm minimization as shown in Lemma~\ref{TNN-minimization}.

\begin{lemma}\label{TNN-minimization}
For any $\alpha,\rho>0$, $\boldsymbol{Z}\in\mathbb{R}^{m\times n}$, and $r\in\mathbb{N}_{+}$ where $r<\min\{m,n\}$, an optimal solution to the truncated nuclear norm minimization problem
\begin{equation}
    \min_{\boldsymbol{X}}~\alpha\|\boldsymbol{X}\|_{r,*}+\frac{\rho}{2}\|\boldsymbol{X}-\boldsymbol{Z}\|_{F}^{2},
    \label{tnn_prob}
\end{equation}
is given by the generalized singular value thresholding \cite{zhang2011penalty,chen2013reduced,lu2015generalized}:
\begin{equation}
    \hat{\boldsymbol{X}}=\mathcal{D}_{r,\alpha/\rho}(\boldsymbol{Z})=\boldsymbol{U}\operatorname{diag}\left([\boldsymbol{\sigma}-\mathbbm{1}_{r}\cdot\alpha/\rho]_{+}\right)\boldsymbol{V}^\top,
    \label{gsvt}
\end{equation}
where $\boldsymbol{U}\operatorname{diag}(\boldsymbol{\sigma})\boldsymbol{V}^{\top}$ is the SVD of $\boldsymbol{Z}$. $[\cdot]_{+}$ denotes the positive truncation at 0 which satisfies $[\sigma-\alpha/\rho]_{+}=\max\{\sigma-\alpha/\rho,0\}$. $\mathbbm{1}_{r}\in\{0,1\}^{\min\{m,n\}}$ is a binary indicator vector whose first $r$ entries are 0 and other entries are 1.
\end{lemma}

Gathering the results of $\boldsymbol{\mathcal{X}}_{1},\boldsymbol{\mathcal{X}}_{2},\boldsymbol{\mathcal{X}}_{3}$ in Eq.~\eqref{computeX}, we can update the variable $\boldsymbol{\mathcal{X}}$ by
\begin{equation}\label{update_x}
    \boldsymbol{\mathcal{X}}^{\ell+1,k+1}:=\sum_{p=1}^{3}\alpha_p\boldsymbol{\mathcal{X}}_{p}.
\end{equation}

\subsubsection{Update Variable $\boldsymbol{Z}$}

Given that $\boldsymbol{\mathcal{X}}=\mathcal{Q}(\boldsymbol{Z})$, we can rewrite Eq.~\eqref{sub2} with respect to $\boldsymbol{Z}$ as follows,
\begin{equation}\label{sub2_formula}
\small
\begin{aligned}
    \boldsymbol{Z}^{\ell+1,k+1}:=&\operatorname{arg}\min_{\boldsymbol{Z}}~\frac{\lambda}{2}\|\boldsymbol{Z}\|_{\boldsymbol{A}^{\ell},\mathcal{H}}+\frac{\rho}{2}\left\|\boldsymbol{\mathcal{X}}^{\ell+1,k+1}-\mathcal{Q}(\boldsymbol{Z})\right\|_{F}^{2}\\
    &-\big\langle\mathcal{Q}(\boldsymbol{Z}),\boldsymbol{\mathcal{T}}^{\ell+1,k}\big\rangle \\
    =&\operatorname{arg}\min_{\boldsymbol{Z}}~\frac{\lambda}{2}\left\|\boldsymbol{Z}\right\|_{\boldsymbol{A}^{\ell},\mathcal{H}}\\
    &+\frac{\rho}{2}\left\|\boldsymbol{Z}-\mathcal{Q}^{-1}(\boldsymbol{\mathcal{X}}^{\ell+1,k+1}+\boldsymbol{\mathcal{T}}^{\ell+1,k}/\rho)\right\|_{F}^{2}. \\
    % &=\operatorname{arg}\min_{\boldsymbol{Z}}~\lambda\sum_{m}\|\boldsymbol{z}_{m,[h_d+1:]}-\boldsymbol{V}_{m}\boldsymbol{a}_{m}\|_{2}^{2}+\frac{\rho}{2}\|\boldsymbol{\mathcal{X}}^{\ell+1}-\mathcal{Q}(\boldsymbol{Z})\|_{F}^{2}-\big\langle\mathcal{Q}(\boldsymbol{Z}),\boldsymbol{\mathcal{T}}^{\ell}\big\rangle, \\
\end{aligned}
\end{equation}

We use the following Lemma~\ref{lemma_ar} to solve this optimization problem.

\begin{lemma}\label{lemma_ar}
For any multivariate time series $\boldsymbol{Z}\in\mathbb{R}^{M\times T}$ which consists of $M$ time series over $T$ consecutive time points, the autoregressive process for any $(m,t)$th element of $\boldsymbol{Z}$ takes
\begin{equation}
{z}_{m,t}\approx\sum_{i=1}^{d}a_{m,i}z_{m,t-h_i},
\end{equation}
with autoregressive coefficient $\boldsymbol{A}\in\mathbb{R}^{M\times d}$ and time lag set $\mathcal{H}=\{h_1,h_2,\ldots,h_d\}$. This autoregressive process also takes the following general formula:
\begin{equation}\label{vec_mat_ar}
\boldsymbol{\Psi}_{0}\boldsymbol{Z}^\top\approx\sum_{i=1}^{d}\boldsymbol{\Psi}_{i}(\boldsymbol{a}_{i}^\top\odot \boldsymbol{Z}^\top)={\boldsymbol{\Psi}}(\boldsymbol{A}^\top\odot \boldsymbol{Z}^\top),
\end{equation}
and for each time series $\boldsymbol{z}_{m}\in\mathbb{R}^{T},\forall m$, we have
\begin{equation}\label{ar_vector_form}
\boldsymbol{\Psi}_{0}\boldsymbol{z}_{m}\approx\sum_{i=1}^{d}a_{m,i}\boldsymbol{\Psi}_{i}\boldsymbol{z}_{m},
\end{equation}
where $\odot$ denotes the Khatri-Rao product, and
\begin{equation*}
\begin{aligned}
\boldsymbol{\Psi}_{0}&=\left[\begin{array}{cc}
  \boldsymbol{0}_{(T-h_d)\times h_d} & \boldsymbol{I}_{T-h_d} \\
\end{array}\right]\in\mathbb{R}^{(T-h_d)\times T}, \\
\boldsymbol{\Psi}_{i}&=\left[\begin{array}{ccc}
  \boldsymbol{0}_{(T-h_d)\times (h_d-h_i)} & \boldsymbol{I}_{T-h_d} & \boldsymbol{0}_{(T-h_d)\times h_i} \\
\end{array}\right]\\
&\in\mathbb{R}^{(T-h_d)\times T},i=1,2,\ldots,d, \\
{\boldsymbol{\Psi}}&=\left[\begin{array}{cccc}
    \boldsymbol{\Psi}_{1} & \boldsymbol{\Psi}_{2} & \cdots & \boldsymbol{\Psi}_{d} \\
\end{array}\right]\in\mathbb{R}^{(T-h_d)\times (dT)}, \\
\end{aligned}
\end{equation*}
are matrices defined based on time lag set $\mathcal{H}$.
\end{lemma}

According to Lemma~\ref{lemma_ar}, there are two options for updating $\boldsymbol{Z}$ when $\boldsymbol{A}$ and $\mathcal{H}$ are known. The first is to minimize the errors in the form of matrix as described in Eq.~\eqref{vec_mat_ar}, and the second is to minimize the errors in the form of vector as described in Eq.~\eqref{ar_vector_form}. The first solution involves complicated operations and possibly high computational cost (see Theorem~\ref{theorem_ar} in Appendix~\ref{supp_theorem} for details). We follow the second approach which takes the vector form for optimizing $\boldsymbol{Z}$. This yields a closed-form solution in Lemma~\ref{lemma_ar_solution}.

\begin{lemma}\label{lemma_ar_solution}
Suppose $\boldsymbol{\Psi}_{0},\boldsymbol{\Psi}_{1},\ldots,\boldsymbol{\Psi}_{d}\in\mathbb{R}^{(T-h_d)\times T}$ and autoregressive coefficient $\boldsymbol{A}\in\mathbb{R}^{M\times d}$ are known as defined in Lemma~\ref{lemma_ar}, then for any $m\in\{1,2,\ldots,M\}$, an optimal solution to the problem
\begin{equation}
    \boldsymbol{z}_{m}:=\operatorname{arg}\min_{\boldsymbol{z}}~\frac{1}{2}\left\|\boldsymbol{\Psi}_{0}\boldsymbol{z}-\sum_{i=1}^{d}a_{m,i}\boldsymbol{\Psi}_{i}\boldsymbol{z}\right\|_{2}^{2}+\frac{\alpha}{2}\|\boldsymbol{z}-\boldsymbol{x}_{m}\|_{2}^{2},
\end{equation}
is given by
\begin{equation}\label{solution_to_z}
    \boldsymbol{z}_{m}:=\alpha(\boldsymbol{B}_{m}^\top\boldsymbol{B}_{m}+\alpha\boldsymbol{I}_{T})^{-1}\boldsymbol{x}_{m},
\end{equation}
where $\boldsymbol{B}_{m}=\boldsymbol{\Psi}_{0}-\sum_{i=1}^{d}a_{m,i}\boldsymbol{\Psi}_{i}$.
\end{lemma}

\noindent\textbf{Remark}. Lemma~\ref{lemma_ar_solution} in fact provides a least squares solution for $\boldsymbol{z}_{m}$. It is also helpful to define $\boldsymbol{B}_{m},m=1,2,\ldots,M$ as sparse matrices and interpret $\boldsymbol{z}_{m}$ as the solution of the following linear equation:
\begin{equation}
    (\boldsymbol{B}_{m}^\top\boldsymbol{B}_{m}+\alpha\boldsymbol{I}_{T})\boldsymbol{z}_{m}=\alpha\boldsymbol{x}_{m}.
\end{equation}
This can help avoid the expensive inverse operation on the $T$-by-$T$ matrix since $T$ is a possibly large value.

According to Lemma~\ref{lemma_ar_solution}, for any $m\in\{1,2,\ldots,M\}$, the closed-form solution to Eq.~\eqref{sub2_formula} is given by
\begin{equation}\label{sub2_solution}
\begin{aligned}
    \boldsymbol{z}_{m}^{\ell+1,k+1}:=&\frac{\rho}{\lambda}\left(\boldsymbol{B}_{m}^\top\boldsymbol{B}_{m}+\frac{\rho}{\lambda}\boldsymbol{I}_{T}\right)^{-1} \\
    &\cdot\mathcal{Q}_{m}^{-1}(\boldsymbol{\mathcal{X}}^{\ell+1,k+1}+\boldsymbol{\mathcal{T}}^{\ell+1,k}/\rho),
\end{aligned}
\end{equation}
where $\boldsymbol{B}_{m}=\boldsymbol{\Psi}_{0}-\sum_{i=1}^{d}a^{\ell}_{m,i}\boldsymbol{\Psi}_{i}$ in which $\boldsymbol{\Psi}_{0},\boldsymbol{\Psi}_{1},\ldots,\boldsymbol{\Psi}_{d}$ follow the same definition as in Lemma~\ref{lemma_ar}.

\subsubsection{Update Variable $\boldsymbol{A}$}

As mentioned above, $\boldsymbol{A}\in\mathbb{R}^{M\times d}$ is the coefficient matrix in the defined temporal variation term. To estimate $\boldsymbol{A}$, we solve the following problem derived from Eq.~\eqref{ar_minimize}:
\begin{equation}
\begin{aligned}
    \boldsymbol{A}^{\ell+1}:&=\operatorname{arg}\min_{\boldsymbol{A}}~\sum_{m,t}(z_{m,t}^{\ell+1,K}-\sum_{i}a_{m,i}z_{m,t-h_i}^{\ell+1,K})^2 \\
    &=\operatorname{arg}\min_{\boldsymbol{A}}~\sum_{m}\left\|\boldsymbol{z}^{\ell+1,K}_{m,[h_d+1:]}-\boldsymbol{V}_{m}\boldsymbol{a}_{m}\right\|_{2}^{2},
\end{aligned}
\end{equation}
where $\boldsymbol{V}_{m}=\left(\boldsymbol{v}_{h_d+1},\cdots,\boldsymbol{v}_{IJ}\right)^\top\in\mathbb{R}^{(IJ-h_d)\times d}$ and $\boldsymbol{v}_{t}=({z}^{\ell+1,K}_{m,t-h_1},\cdots,{z}^{\ell+1,K}_{m,t-h_d})^\top\in\mathbb{R}^{d},t=h_d+1,\ldots,IJ$ are formed by ${\boldsymbol{Z}}^{\ell+1,K}$. Obviously, this optimization has a closed-form solution, which is given by
\begin{equation}\label{sub3_solution}
    \boldsymbol{a}_{m}^{\ell+1}:=\boldsymbol{V}_{m}^{\dagger}\boldsymbol{z}^{\ell+1,K}_{m,[h_d+1:]},\forall m,
\end{equation}
where $\cdot^{\dagger}$ denotes the pseudo-inverse.

%\subsection{Overall Solution Algorithm for LATC}

Algorithm~\ref{imputer} shows the overall algorithm for solving LATC. The algorithm has three parameters $\rho$, $\lambda$ and $r$. Parameter $\rho$ controls the ADMM and the singular value thresholding. Parameter $\lambda$ is a trade-off between truncated nuclear norm and temporal variation, which can be typically set to $\lambda=c\cdot\rho$. Thus, $c=1$ implies that these two norms have the same importance in the objective. The recovered matrix is computed by $\hat{\boldsymbol{X}}^{\ell}=\mathcal{Q}^{-1}(\boldsymbol{\mathcal{X}}^{\ell,K})$ at each outer iteration. The algorithm returns the converged $\hat{\boldsymbol{X}}$ as the final result, if the convergence criteria is met.

\begin{algorithm}[!tb]
\caption{$\text{imputer}(\boldsymbol{Y},\mathcal{H},\rho,\lambda,r)$}
\label{imputer}
Initialize $\boldsymbol{\mathcal{T}}^{0,0}$ as zeros and $\boldsymbol{A}^{0}$ as small random values. Set $\mathcal{P}_{\Omega}(\boldsymbol{Z}^{0,0})=\mathcal{P}_{\Omega}(\boldsymbol{Y})$, $\alpha_1=\alpha_2=\alpha_3=\frac{1}{3}$, $K=3$, and $\ell=0$. \\
\While{not converged}{
\For{$k=0$ \KwTo $K-1$}{
$\rho=\min\{1.05\times\rho, \rho_{\text{max}}\}$; \\
\For{$j=1$ \KwTo $J$}{
Compute $\boldsymbol{\mathcal{X}}_{j}$ by Eq.~\eqref{computeX};
}
Update $\boldsymbol{\mathcal{X}}^{\ell+1,k+1}$ by Eq.~\eqref{update_x}; \\
\For{$m=1$ \KwTo $M$}{
Update $\boldsymbol{z}_{m}^{\ell+1,k+1}$ by Eq.~\eqref{sub2_solution};
}
Update $\boldsymbol{\mathcal{T}}^{\ell+1,k+1}$ by Eq.~\eqref{sub3}; \\
Transform observation information by letting $\mathcal{P}_{\Omega}(\boldsymbol{Z}^{\ell+1,k+1})=\mathcal{P}_{\Omega}(\boldsymbol{Y})$;
}
\For{$m=1$ \KwTo $M$}{
Update $\boldsymbol{a}_{m}^{\ell+1}$ by Eq.~\eqref{sub3_solution};
}
$\ell:=\ell+1$;
}

\Return recovered matrix $\hat{\boldsymbol{X}}$.
\end{algorithm}

\section{Experiments}\label{experiments}

In this section, we evaluate the proposed LATC model on several real-world traffic data sets with different missing patterns.

\subsection{Traffic Data Sets}

We use the following four spatiotemporal traffic sets for our benchmark experiment.
\begin{itemize}
    \item \textbf{(G)}: Guangzhou urban traffic speed data set.\footnote{\url{https://doi.org/10.5281/zenodo.1205229}} This data set contains traffic speed collected from 214 road segments over two months (from August 1 to September 30, 2016) with a 10-minute resolution (i.e., 144 time intervals per day) in Guangzhou, China. The prepared data is of size $214\times 8784$ in the form of multivariate time series matrix (or tensor of size $214\times 144\times 61$).
    \item \textbf{(H)}: Hangzhou metro passenger flow data set.\footnote{\url{https://tianchi.aliyun.com/competition/entrance/231708/information}} This data set provides incoming passenger flow of 80 metro stations over 25 days (from January 1 to January 25, 2019) with a 10-minute resolution in Hangzhou, China. We discard the interval 0:00 a.m. 6:00 a.m. with no services, and only consider the remaining 108 time intervals of a day. The prepared data is of size $80\times 2700$ in the form of multivariate time series (or tensor of size $80\times 108\times 25$).
    \item \textbf{(S)}: Seattle freeway traffic speed data set.\footnote{\url{https://github.com/zhiyongc/Seattle-Loop-Data}} This data set contains freeway traffic speed from 323 loop detectors with a 5-minute resolution (i.e., 288 time intervals per day) over the first four weeks of January, 2015 in Seattle, USA. The prepared data is of size $323\times 8064$ in the form of multivariate time series (or tensor of size $323\times 288\times 28$).
    \item \textbf{(P)}: Portland highway traffic volume data set.\footnote{\url{https://portal.its.pdx.edu/home}} This data set is collected from highways in the Portland-Vancouver Metropolitan region, which contains traffic volume from 1156 loop detectors with a 15-minute resolution (i.e., 96 time intervals per day) in January, 2021. The prepared data is of size $1156\times 2976$ in the form of multivariate time series matrix (or tensor of size $1156\times 96\times 31$).
\end{itemize}

Note that the adapted data sets and Python codes for our experiments are available on Github.\footnote{\url{https://github.com/xinychen/transdim}}

\subsection{Missing Data Generation}

To evaluate the performance of LATC for missing traffic data imputation thoroughly, we take into account three missing data patterns as shown in Fig.~\ref{missing_patterns}, i.e., random missing (RM), non-random missing (NM), and blackout missing (BM). RM and NM data are generated by referring to our prior work \cite{chen2019abayesian}. According to the mechanism of RM and NM data, we mask certain amount of observations as missing values (e.g., 30\%, 70\%, 90\%), and the remaining partial observations are input data for learning a well-behaved model. BM pattern is different from RM and NM patterns, which masks observations of all spatial sensors/locations as missing values with certain window length. BM is a challenging scenario with complete column-wise missing. We set the missing rate in the following experiments to 30\%.

\begin{figure}[!t]
\begin{center}
\subfigure[Random missing (RM).]{
\resizebox{6.5cm}{!}{
\begin{tikzpicture}[domain=0:110]
\begin{axis}[
  height=3cm,
  width=7cm,
  axis x line=center,
  axis y line=center,
  xtick={15,30,45,60,75,90,105},
  ytick={0},
  xticklabels={\scriptsize$I$,\scriptsize$2I$,\scriptsize$3I$,\scriptsize$4I$,\scriptsize$5I$,\scriptsize$6I$,\scriptsize$7I$},
  yticklabels={},
  xlabel style={below left},
  ylabel style={below right},
  xmin=1,
  xmax=115,
  ymin=0,
  ymax=60]
\addplot[no marks,smooth,draw=cyan!60!blue,thick] file {speed1.data};
\addplot+[only marks,
    mark=*,
    mark options={scale=0.9,fill=white},
    draw=gray!60,thick] plot coordinates{(6,36.90) (13,24.89) (25,39.41) (29,39.85) (55,32.57) (76,41.60) (83,34.34) (89,33.23) (95,43.76)};
\addplot[no marks,smooth,draw=orange!60!black,thick] file {speedr1.data};
\addplot+[only marks,
    mark=*,
    mark options={scale=0.9,fill=white},
    draw=gray!60,thick] plot coordinates{(14,29) (19,34) (39,24) (45,27) (70,21) (83,22)};
\end{axis}
\end{tikzpicture}}
}

\vspace{-0.5em}

\subfigure[Non-random missing (NM).]{
\resizebox{6.5cm}{!}{
\begin{tikzpicture}[domain=0:110]
\begin{axis}[
  height=3cm,
  width=7cm,
  axis x line=center,
  axis y line=center,
  xtick={15,30,45,60,75,90,105},
  ytick={0},
  xticklabels={\scriptsize$I$,\scriptsize$2I$,\scriptsize$3I$,\scriptsize$4I$,\scriptsize$5I$,\scriptsize$6I$,\scriptsize$7I$},
  yticklabels={},
  xlabel style={below left},
  ylabel style={below right},
  xmin=1,
  xmax=115,
  ymin=0,
  ymax=60]
% \addplot[no marks,smooth,draw=cyan!60!blue,thick] file {speed1.data};
\addplot[no marks,smooth,draw=gray!30,thick] file {speedn1.data};
\addplot[no marks,smooth,draw=cyan!60!blue,thick] file {speedn2.data};
\addplot[no marks,smooth,draw=gray!30,thick] file {speedn3.data};
\addplot[no marks,smooth,draw=cyan!60!blue,thick] file {speedn4.data};
\addplot[no marks,smooth,draw=cyan!60!blue,thick] file {speedn5.data};
\addplot[no marks,smooth,draw=gray!30,thick] file {speedn6.data};
\addplot[no marks,smooth,draw=cyan!60!blue,thick] file {speedn7.data};
\addplot[no marks,smooth,draw=orange!60!black,thick] file {speedn01.data};
\addplot[no marks,smooth,draw=gray!30,thick] file {speedn02.data};
\addplot[no marks,smooth,draw=orange!60!black,thick] file {speedn03.data};
\end{axis}
\end{tikzpicture}}}

\vspace{-0.5em}

\subfigure[Blackout missing (BM).]{
\resizebox{6.5cm}{!}{
\begin{tikzpicture}[domain=0:110]
\begin{axis}[
  height=3cm,
  width=7cm,
  axis x line=center,
  axis y line=center,
  xtick={15,30,45,60,75,90,105},
  ytick={0},
  xticklabels={\scriptsize$I$,\scriptsize$2I$,\scriptsize$3I$,\scriptsize$4I$,\scriptsize$5I$,\scriptsize$6I$,\scriptsize$7I$},
  yticklabels={},
  xlabel style={below left},
  ylabel style={below right},
  xmin=1,
  xmax=115,
  ymin=0,
  ymax=60]
% \addplot[no marks,smooth,draw=cyan!60!blue,thick] file {speed1.data};
\addplot[no marks,smooth,draw=cyan!60!blue,thick] file {speed1.data};
\addplot[no marks,smooth,draw=gray!50,thick] file {speedb2.data};
\addplot[no marks,smooth,draw=gray!50,thick] file {speedb3.data};
\addplot[no marks,smooth,draw=gray!50,thick] file {speedb4.data};
\addplot[no marks,smooth,draw=orange!60!black,thick] file {speedr1.data};
\addplot[no marks,smooth,draw=gray!30,thick] file {speedb02.data};
\addplot[no marks,smooth,draw=gray!30,thick,thick] file {speedb03.data};
\addplot[no marks,smooth,draw=gray!30,thick,thick] file {speedb04.data};
\end{axis}
\end{tikzpicture}}}
\end{center}
\vspace{-1em}
\caption{Illustration of three missing data patterns for spatiotemporal traffic data (e.g., traffic speed). Each time series represent the collected data from a given sensor. In these graphics, two curves correspond to two different time series. (a) Data are missing at random. Small circles indicate the missing values. (b) Data are missing continuously during a few time periods. Segments in gray indicate missing values. (c) No sensors are available (i.e., blackout) over a certain time window. }
\label{missing_patterns}
\end{figure}
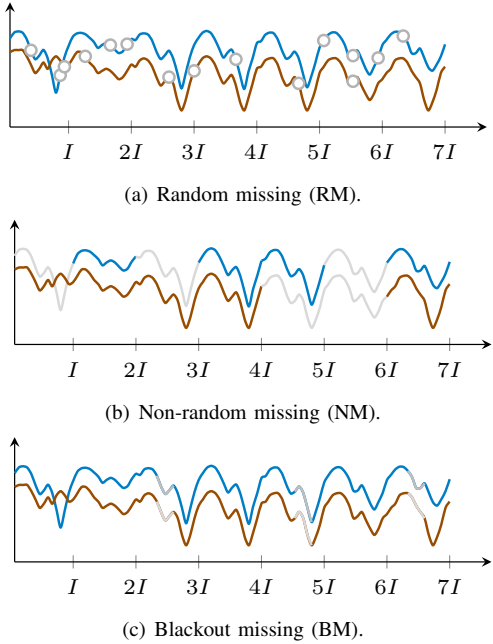

To assess the imputation performance, we use the actual values of the masked missing entries as the ground truth to compute  MAPE and RMSE:
\begin{equation}
\begin{split}
    \text{MAPE}&=\frac{1}{n}\sum_{i=1}^{n}\left|\frac{y_i-\hat{y}_i}{y_i}\right|\times 100, \\
    \text{RMSE}&=\sqrt{\frac{1}{n}\sum_{i=1}^{n}(y_i-\hat{y}_{i})^2},
\end{split}
\end{equation}
where $y_i$ and $\hat{y}_{i}$ are actual values and imputed values, respectively.

\subsection{Baseline Models}

For comparison, we take into account the following baseline:
\begin{itemize}
    \item Low-Rank Autoregressive Matrix Completion (LAMC). This is a matrix-form variant of the LATC model.
    \item Low-Rank Tensor Completion with Truncation Nuclear Norm minimization (LRTC-TNN, \cite{chen2020anonconvex}). This is a low-rank completion model in which truncated nuclear norm minimization can help maintain the most important low-rank patterns. Since the truncation in LRTC-TNN is a defined as a rate parameter, we adapt LRTC-TNN to use integer truncation in order to make it consistent with LATC.
    \item Bayesian Temporal Matrix Factorization (BTMF, \cite{chen2021bayesian}). This is a fully Bayesian temporal factorization framework which builds the correlation of temporal dynamics on latent factors by vector autoregressive process. Due to the temporal modeling, it outperforms the standard matrix factorization in the missing data imputation tasks\cite{chen2021bayesian}.
    \item Smooth PARAFAC Tensor Completion (SPC, \cite{yokota2016smooth}). This is a tensor decomposition based completion model with total variation smoothness constraints.
\end{itemize}

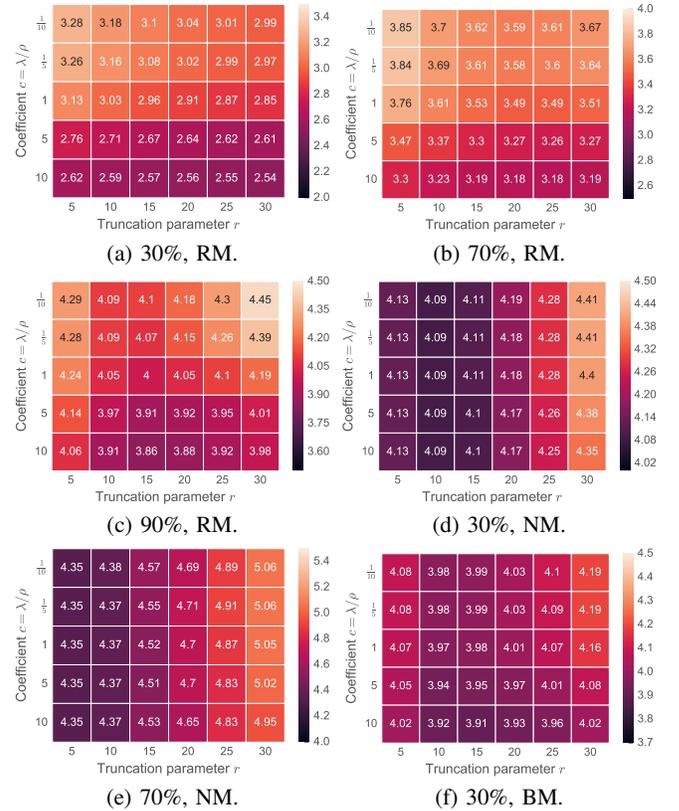
\begin{figure}[!t]
\begin{center}
\resizebox{9cm}{!}{
\begin{tikzpicture}
\pgfdeclareimage[height = 4.5cm]{snap1}{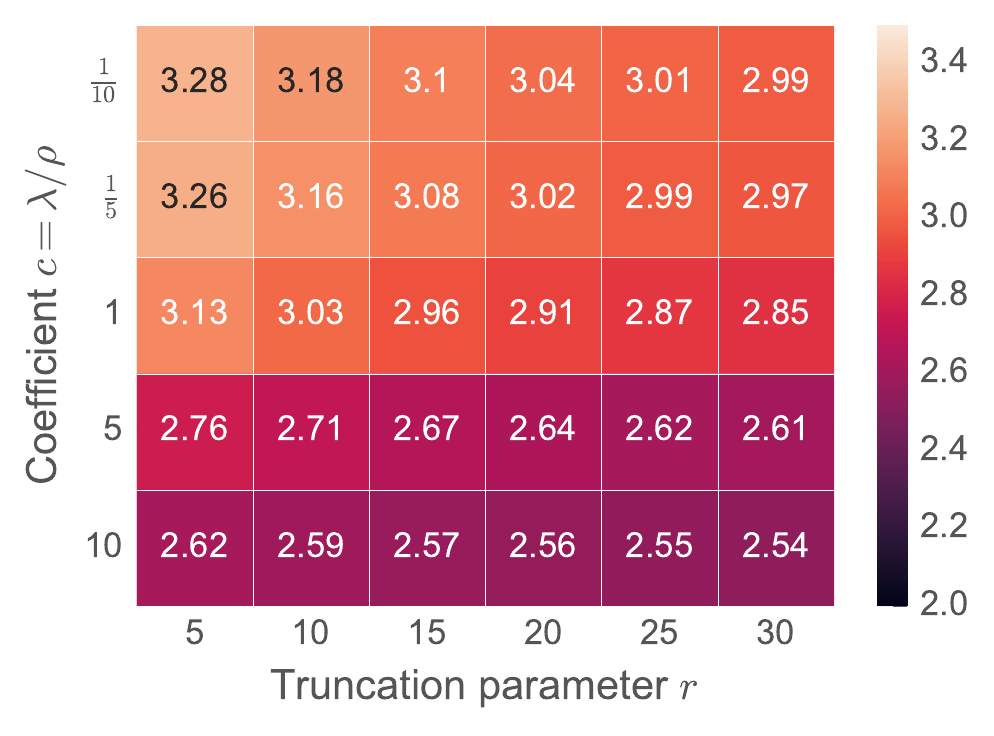}
\node (snap1) at (0, 0) {\pgfuseimage{snap1}};
\draw (0, -2.5) node {\large{(a) 30\%, RM.}};

\pgfdeclareimage[height = 4.5cm]{snap2}{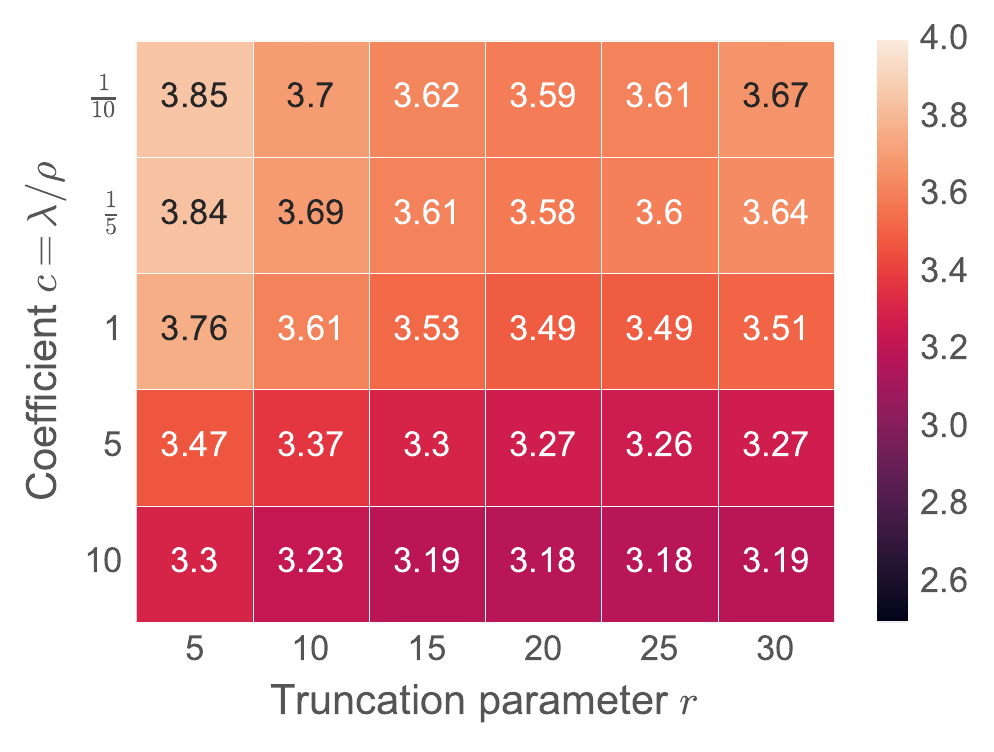}
\node (snap2) at (6, 0) {\pgfuseimage{snap2}};
\draw (6, -2.5) node {\large{(b) 70\%, RM.}};

\pgfdeclareimage[height = 4.5cm]{snap1}{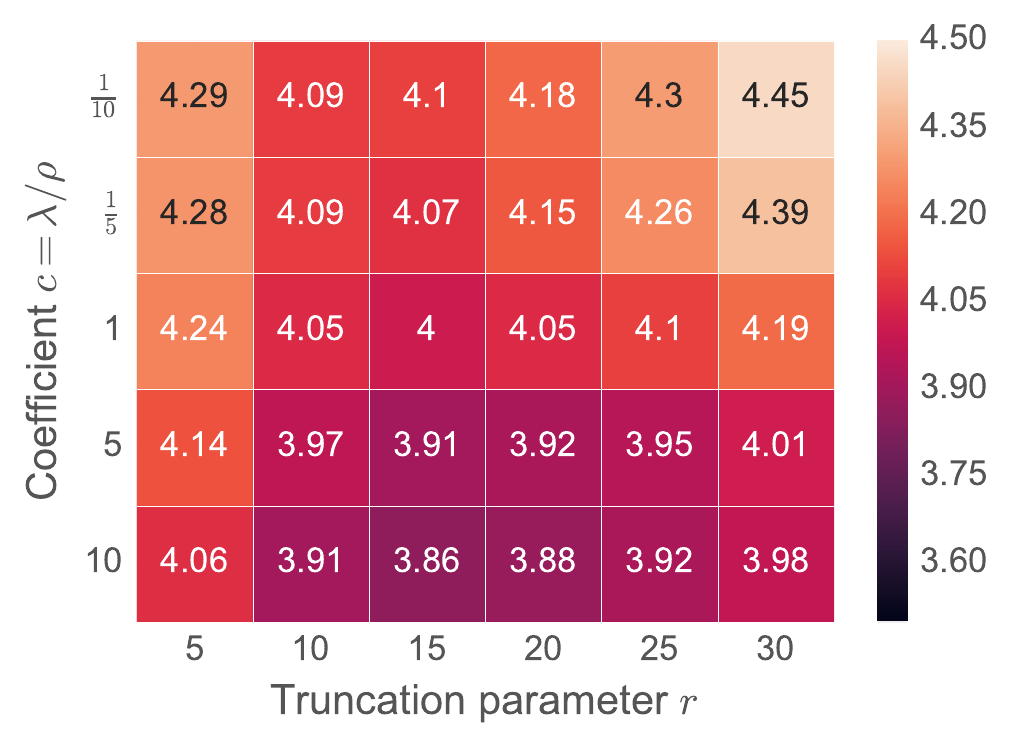}
\node (snap1) at (0, -5) {\pgfuseimage{snap1}};
\draw (0, -2.5-5) node {\large{(c) 90\%, RM.}};

\pgfdeclareimage[height = 4.5cm]{snap2}{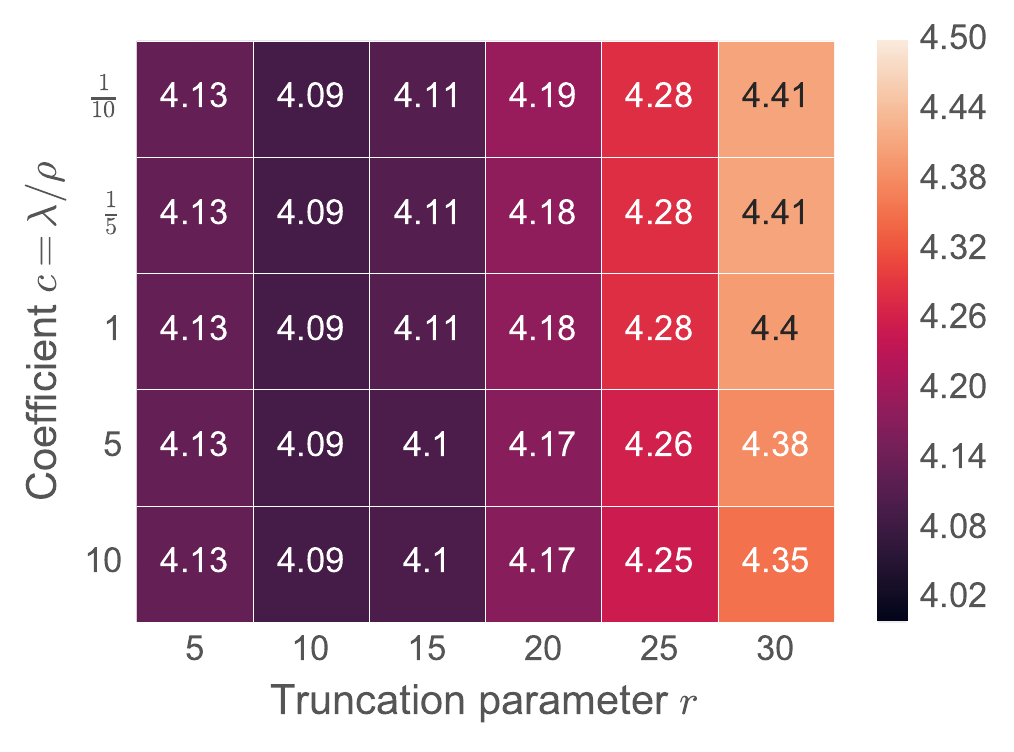}
\node (snap2) at (6, -5) {\pgfuseimage{snap2}};
\draw (6, -2.5-5) node {\large{(d) 30\%, NM.}};

\pgfdeclareimage[height = 4.5cm]{snap1}{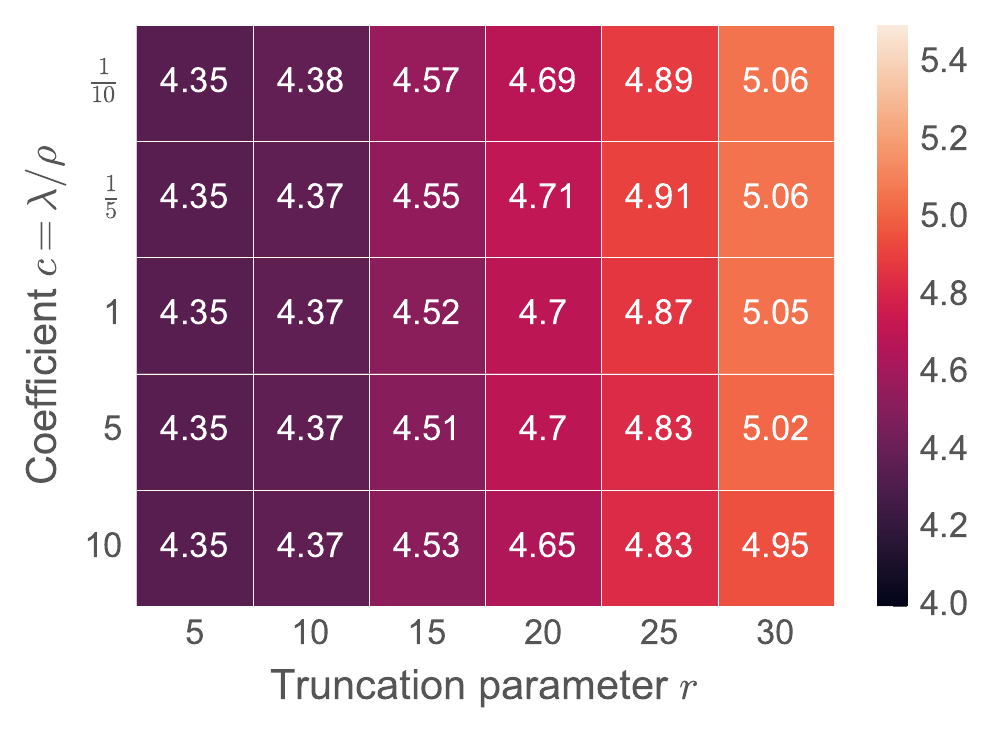}
\node (snap1) at (0, -10) {\pgfuseimage{snap1}};
\draw (0, -2.5-10) node {\large{(e) 70\%, NM.}};

\pgfdeclareimage[height = 4.5cm]{snap2}{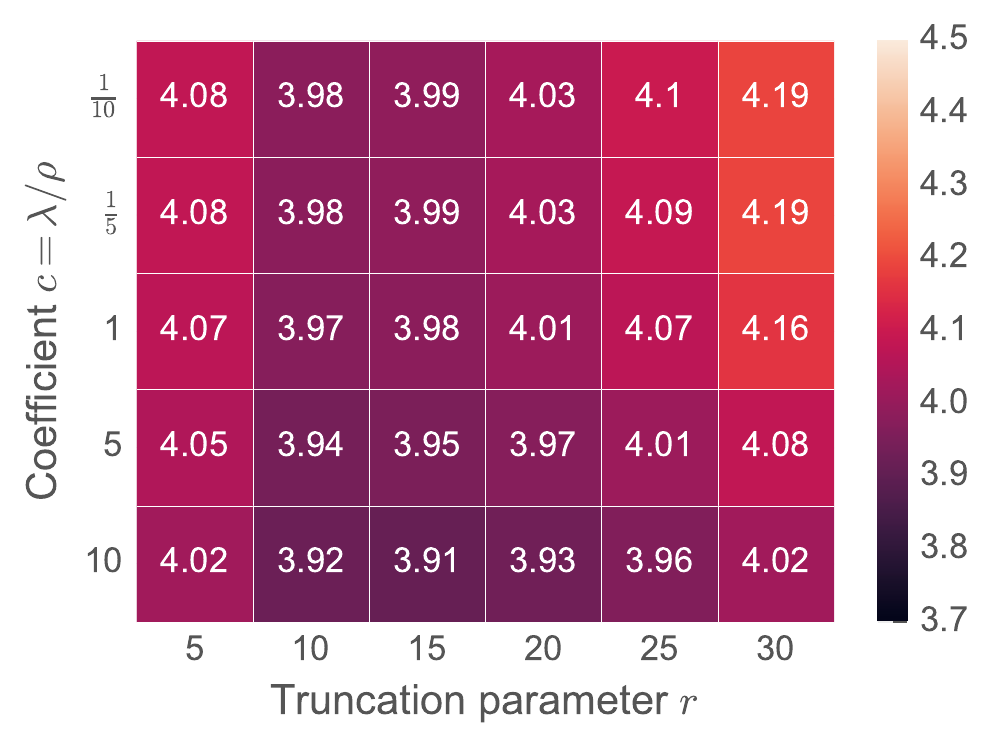}
\node (snap2) at (6, -10) {\pgfuseimage{snap2}};
\draw (6, -2.5-10) node {\large{(f) 30\%, BM.}};
\end{tikzpicture}}
\end{center}
\vspace{-0.5em}
\caption{RMSEs of LATC imputation on Guangzhou urban traffic speed data where $\rho=1\times10^{-4}$ for RM data and $\rho=1\times10^{-5}$ for NM/BM data. The smallest RMSE is achieved by: (a) $c=10,r=30$; (b) $c=10,r=20,25$; (c) $c=10,r=15$; (d) $r=10$; (e) $r=5$; (f) $c=10,r=15$.}
\label{rmse_heatmap_gdata}
\end{figure}

% Since LRTC-TNN(R) and BTMF, as the most recent studies, have been demonstrated to be superior to some state-of-the-art matrix/tensor models for missing data imputation on some traffic data sets, also including three data sets used in this work. We do not include the evaluated inferior baseline models in the following.

\subsection{Results}

There are several parameters in LATC, including learning rate $\rho$, weight parameter $\lambda$, truncation $r$, and time lag set $\mathcal{H}$. The most important parameters are the coefficient $c=\lambda/\rho$ and the truncation $r$. For other parameters including $\rho$ and time lag set $\mathcal{H}$, we conduct preliminary test for choosing them. $\rho$ is chosen from $\{1\times10^{-5},1\times10^{-4}\}$ for all data sets. To assess the sensitivity of the model over $c$ and $r$, we develop the following setting for our imputation experiments:
\begin{itemize}
    \item Time lag set is set as $\{1,2,\ldots,6\}$ for (G), (H), and (S) data, and $\{1,2,3,4\}$ for (P) data;
    \item $\lambda=c\cdot\rho$ where $c\in\{\frac{1}{10}\frac{1}{5},1,5,10\}$;
    \item $r\in\{5,10,15,20,25,30\}$ and $r<\min\{M,I,J\}$.
\end{itemize}

%We particularly take into account different values of $c$ and $r$.

Fig.~\ref{rmse_heatmap_gdata} shows the heatmaps of imputation RMSE values achieved by LATC model on Guangzhou urban traffic speed data. It demonstrates that: 1) for RM and BM data, when $c=10$, LATC model achieves the best imputation performance and the truncation $r$ has little impact on the final results; 2) for NM data, the coefficient $c$ is less important than the truncation $r$. LATC model achieves the best performance when the truncation is a relatively small value (e.g., 5, 10). These results verifies the importance of temporal variation minimization for RM and BM imputation.

Fig.~\ref{rmse_heatmap_hdata} shows similar heatmaps for Hangzhou metro passenger flow data. It can be seen that: 1) for RM and BM data, when $c=1$, LATC model achieves the best imputation performance; 2) for NM data, LATC model achieves the best performance with small coefficient $c$ and truncation $r$ (e.g., 5).

\begin{figure}[!t]
\begin{center}
\resizebox{9cm}{!}{
\begin{tikzpicture}
\pgfdeclareimage[height = 4.5cm]{snap1}{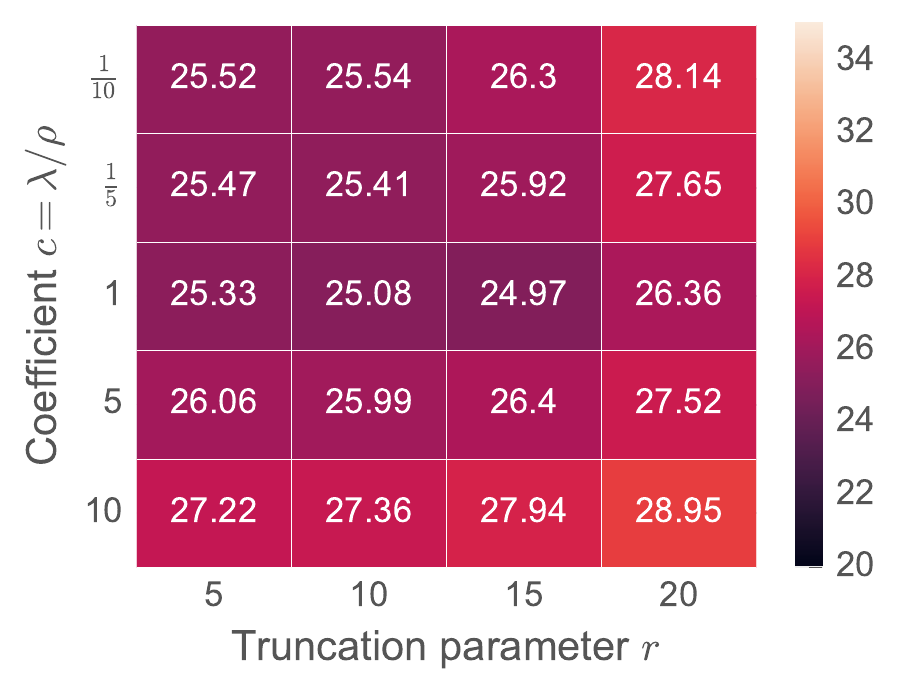}
\node (snap1) at (0, 0) {\pgfuseimage{snap1}};
\draw (0, -2.5) node {\large{(a) 30\%, RM.}};

\pgfdeclareimage[height = 4.5cm]{snap2}{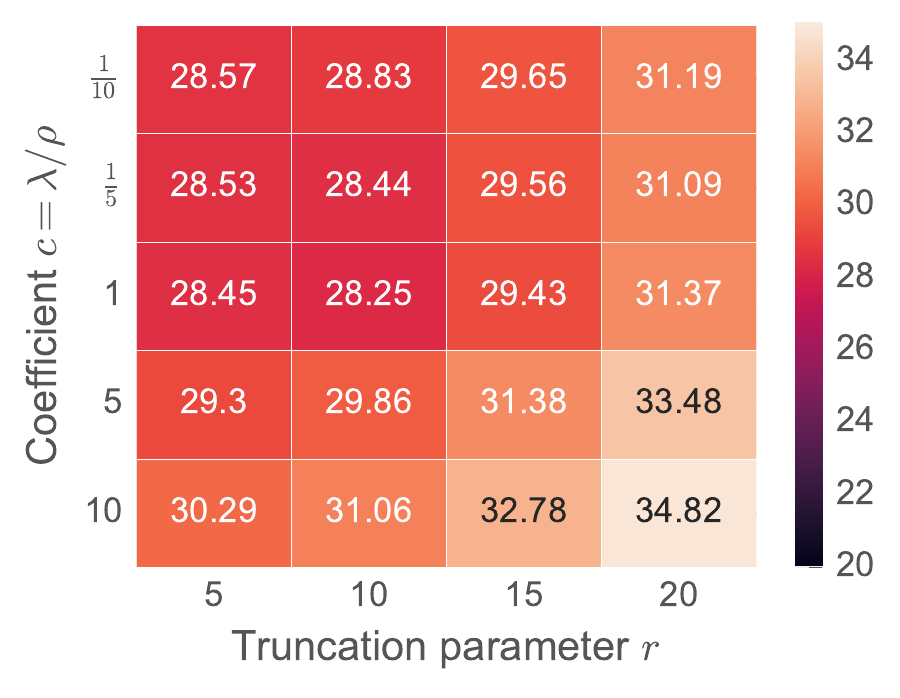}
\node (snap2) at (6, 0) {\pgfuseimage{snap2}};
\draw (6, -2.5) node {\large{(b) 70\%, RM.}};

\pgfdeclareimage[height = 4.5cm]{snap1}{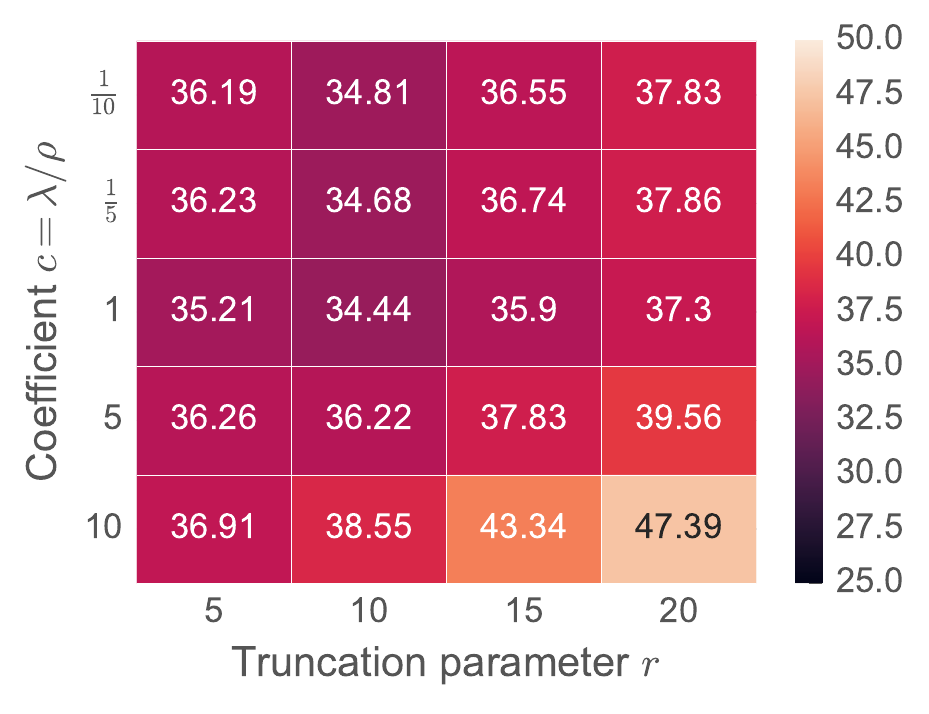}
\node (snap1) at (0, -5) {\pgfuseimage{snap1}};
\draw (0, -2.5-5) node {\large{(c) 90\%, RM.}};

\pgfdeclareimage[height = 4.5cm]{snap2}{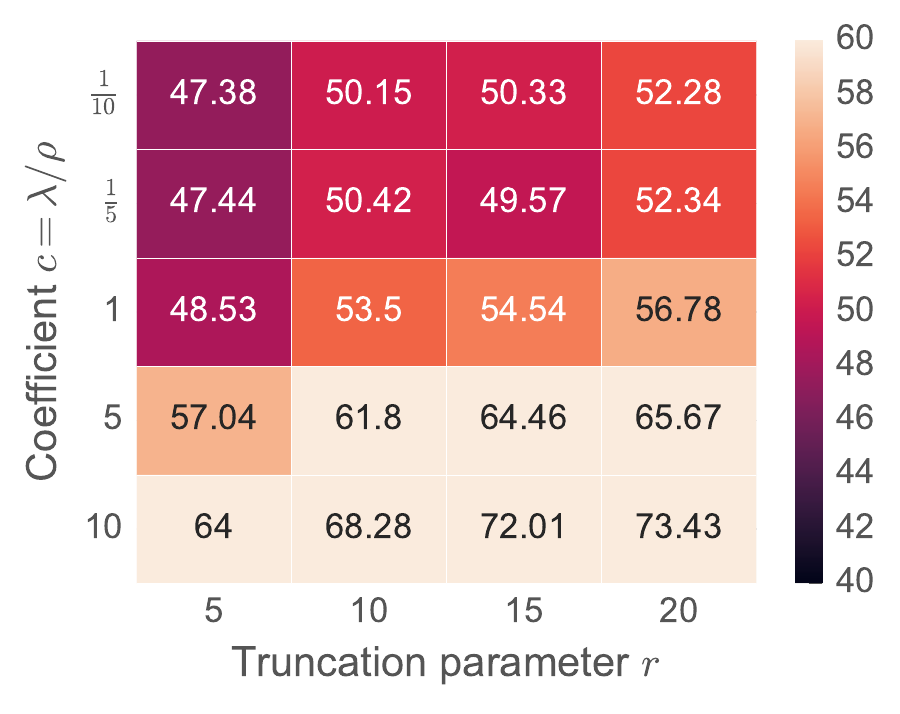}
\node (snap2) at (6, -5) {\pgfuseimage{snap2}};
\draw (6, -2.5-5) node {\large{(d) 30\%, NM.}};

\pgfdeclareimage[height = 4.5cm]{snap1}{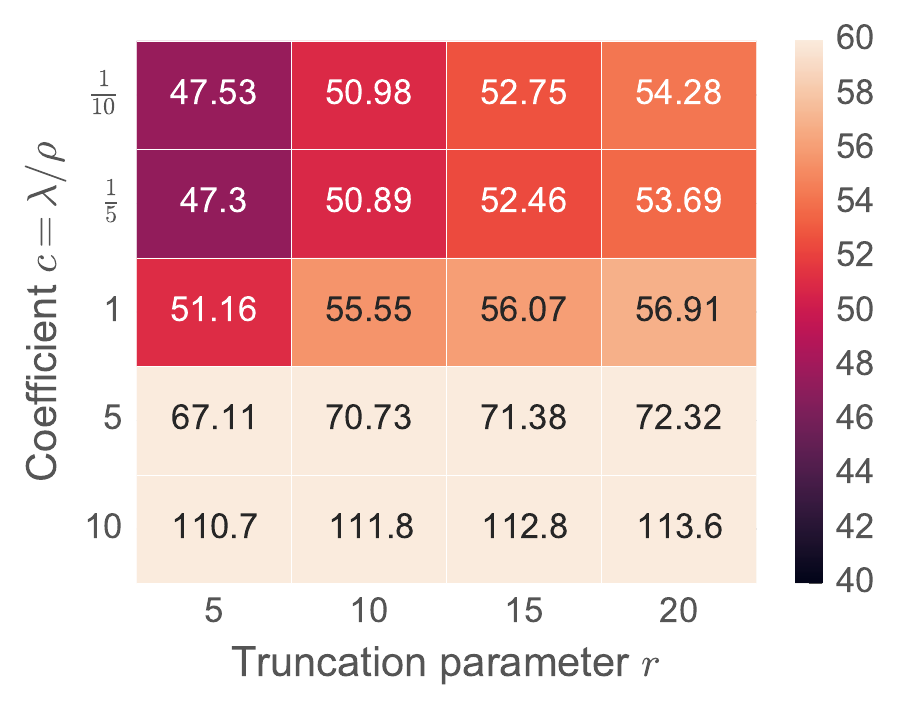}
\node (snap1) at (0, -10) {\pgfuseimage{snap1}};
\draw (0, -2.5-10) node {\large{(e) 70\%, NM.}};

\pgfdeclareimage[height = 4.5cm]{snap2}{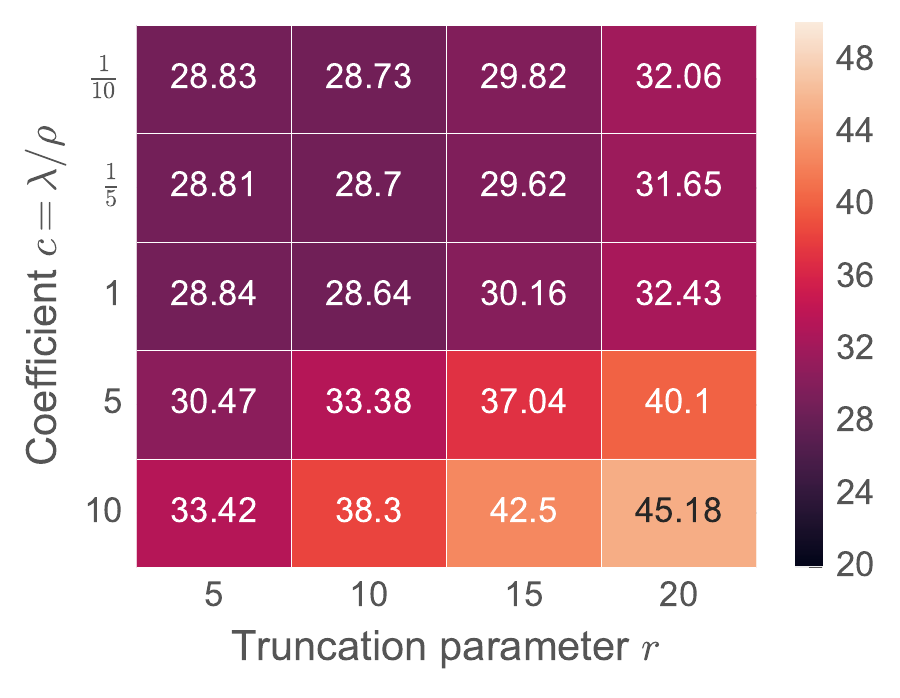}
\node (snap2) at (6, -10) {\pgfuseimage{snap2}};
\draw (6, -2.5-10) node {\large{(f) 30\%, BM.}};
\end{tikzpicture}}
\end{center}
\vspace{-0.5em}
\caption{RMSEs of LATC imputation on Hangzhou metro passenger flow data ($\rho=1\times10^{-5}$). The smallest RMSE is achieved by: (a) $c=1,r=15$; (b-c) $c=1,r=10$; (d) $c=\frac{1}{10},r=5$; (e) $c=\frac{1}{5},r=5$; (f) $c=1,r=10$.}
\label{rmse_heatmap_hdata}
\end{figure}
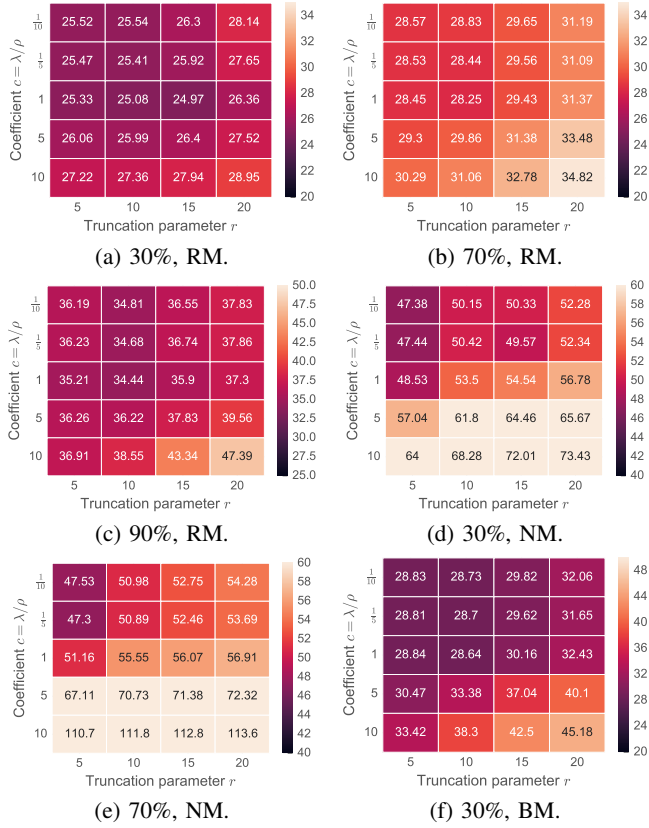

% From Fig.~\ref{rmse_heatmap_sdata}, it can be seen that: 1) the coefficient $c$ has little impact on the final imputation results for all cases; 2) greater truncation can result in better imputation performance in the case of RM and NM data; 3) for BM data, as $c=10$, LATC model achieves the best imputation performance. These results demonstrate the positive influence of temporal variation in LATC.

By testing the LATC model in the similar way, it can indicate the importance of temporal variation on other two data sets. On Seattle freeway traffic speed data, we observe that the coefficient $c$ has little impact on the final imputation for the RM and NM data. However, there show the positive influence of temporal variation in LATC for BM data. On Portland highway traffic volume data, a relatively large coefficient $c$ (e.g., 5 and 10) can make the model less sensitive to the various truncation values for RM and BM data.

As mentioned above, despite the truncated nuclear norm built on tensor, the results also show the advantage of temporal variation built on the multivariate time series matrix. Due to the temporal modeling, temporal variation can improve the imputation performance for missing traffic data imputation. Table~\ref{table1} shows the overall imputation performance of LATC and baseline models on the four selected traffic data sets with various missing scenarios. Of these results, NM and BM data seem to be more difficult to reconstruct with all these imputation models than RM data. In most cases, LATC outperforms other baseline models. Comparing LATC with LAMC shows the advantage of tensor structure, i.e., LATC with tensor structure performs better than LAMC with matrix structure. Comparing LATC with LRTC-TNN shows the advantage of temporal variation, i.e., temporal modeling with autoregressive process has positive influence for improving the imputation performance. For volume data sets (H) and (P), the relative errors are quite high because some volume values are close to 0 or relatively small and estimating these values would accumulate relatively large relative errors.

Figs.~\ref{Hangzhou-imputation-plot}, \ref{seattle-imputation-plot}, and \ref{portland-imputation-plot} show some imputation examples with different missing scenarios that achieved by LATC. In these examples, we can see explicit temporal dependencies underlying traffic time series data. For all missing scenarios, LATC can achieve accurate imputation and learn the true signals from observations even with severe missing data (e.g., NM/BM data). In Fig.~\ref{Hangzhou-imputation-plot}, it shows that the time series signal of passenger flow is not complex. By referring to Table~\ref{table1}, we can see that LRTC-TNN without temporal variation outperforms the proposed LATC model on Hangzhou metro passenger flow data, and this demonstrates that not all multivariate time series imputation cases require temporal modeling, for some cases that the signal does not show strong temporal dependencies, purely low-rank model can also provide accurate imputation.

\begin{table*}[!ht]
\caption{Performance comparison (in MAPE/RMSE) of LATC and baseline models for RM, NM, and BM data imputation.}
\label{table1}
\centering
\footnotesize
\begin{tabular}{ll|rrrrrrrr}
\toprule
Data & Missing & LATC & LAMC & LRTC-TNN & BTMF & SPC \\
\midrule
\multirow{6}{*}{(G)} & 30\%, RM & \textbf{5.71}/\textbf{2.54} & 9.51/4.04 & 6.99/3.00 & 7.54/3.27 & 7.37/5.06 \\
& 70\%, RM & \textbf{7.22}/\textbf{3.18} & 10.40/4.37 & 8.38/3.59 & 8.75/3.73 & 8.91/4.44 \\
& 90\%, RM & \textbf{9.11}/\textbf{3.86} & 11.65/4.79 & 9.55/4.05 & 10.02/4.21 & 10.60/4.85 \\
& 30\%, NM & 9.63/4.09 & 10.11/4.23 & 9.61/\textbf{4.07} & 10.32/4.33 & \textbf{9.13}/5.29 \\
& 70\%, NM & 10.37/4.35 & 11.15/4.60 & \textbf{10.36}/\textbf{4.34} & 11.36/4.85 & 11.15/5.17 \\
& 30\%, BM-6 & \textbf{9.23}/\textbf{3.91} & 12.15/5.17 & 9.45/3.97 & 12.43/7.04 & 11.14/5.13 \\
\midrule
\multirow{6}{*}{(H)} & 30\%, RM & 19.12/24.97 & 22.65/42.94 & \textbf{18.87}/\textbf{24.90} & 22.37/28.66 & 19.82/26.21 \\
& 70\%, RM & 20.25/28.25 & 25.30/51.26 & \textbf{20.07}/\textbf{28.13} & 25.65/32.23 & 21.02/31.91 \\
& 90\%, RM & 24.32/\textbf{34.44} & 32.30/66.13 & \textbf{23.46}/35.84 & 31.51/46.24 & 24.97/49.68 \\
& 30\%, NM & \textbf{19.93}/\textbf{47.38} & 22.93/67.08 & 19.94/50.12 & 25.61/77.00 & 27.46/68.56 \\
& 70\%, NM & 24.30/47.30 & 29.23/63.95 & \textbf{23.88}/\textbf{45.06} & 34.50/70.11 & 46.86/98.81 \\
& 30\%, BM-6 & 21.93/28.64 & 30.78/66.03 & \textbf{21.40}/\textbf{27.83} & 52.15/57.61 & 22.49/37.53 \\
\midrule
\multirow{6}{*}{(S)} & 30\%, RM & \textbf{4.90}/\textbf{3.16} & 5.98/3.73 & 4.99/3.20 & 5.91/3.72 & 5.92/3.62 \\
& 70\%, RM & \textbf{5.96}/\textbf{3.71} & 8.02/4.70 & 6.10/3.77 & 6.47/3.98 & 7.38/4.30 \\
& 90\%, RM & \textbf{7.47}/\textbf{4.51} & 10.56/5.91 & 8.08/4.80 & 8.17/4.81 & 9.75/5.31 \\
& 30\%, NM & 7.11/4.33 & 6.99/4.25 & \textbf{6.85}/\textbf{4.21} & 9.26/5.36 & 8.87/4.99 \\
& 70\%, NM & 9.46/5.42 & 9.75/5.60 & \textbf{9.23}/\textbf{5.35} & 10.47/6.15 & 11.32/5.92 \\
& 30\%, BM-12 & \textbf{9.44}/\textbf{5.36} & 27.05/13.66 & 9.52/5.41 & 14.33/13.60 & 11.30/5.84 \\
\midrule
\multirow{6}{*}{(P)} & 30\%, RM & 17.46/\textbf{15.89} & 17.93/16.03 & \textbf{17.27}/16.08 & 18.22/19.14 & 21.29/56.73 \\
& 70\%, RM & \textbf{19.56}/\textbf{18.70} & 21.26/19.37 & 19.99/18.73 & 19.96/22.21 & 24.35/43.32 \\
& 90\%, RM & 23.47/22.74 & 25.64/23.75 & \textbf{22.90}/\textbf{22.68} & 23.90/25.71 & 28.45/39.65 \\
& 30\%, NM & \textbf{18.90}/\textbf{18.84} & 19.93/19.69 & 19.59/18.91 & 19.55/20.38 & 26.96/60.33 \\
& 70\%, NM & 24.67/31.74 & 25.75/28.25 & 30.26/60.85 & \textbf{23.86}/\textbf{26.74} & 33.42/47.34\\
& 30\%, BM-4 & \textbf{24.04}/\textbf{23.52} & 29.21/27.60 & 31.74/74.42 & 27.85/25.68 & 31.01/60.33 \\
\bottomrule
\multicolumn{10}{l}{{Best results are highlighted in bold fonts. The number next to the BM denotes the window length.}}
\end{tabular}
\end{table*}

\begin{figure*}[!t]
\centering
\includegraphics[width=0.9\textwidth]{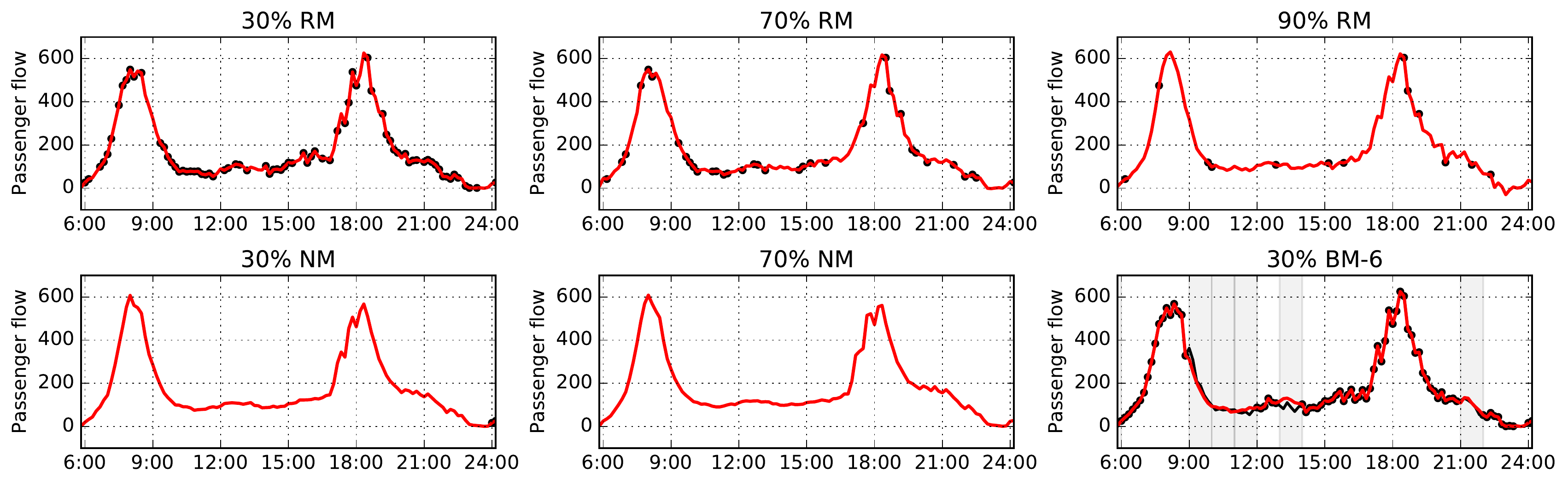}
\caption{Imputed values by LATC for Hangzhou metro passenger flow data. This example corresponds to  metro station \#3 and the 4th day of the data set. Black dots/curves indicate the partially observed data, gray rectangles indicate blackout missing, while red curves indicate the imputed values.}
\label{Hangzhou-imputation-plot}
\end{figure*}

\begin{figure*}[!t]
\centering
\includegraphics[width=0.9\textwidth]{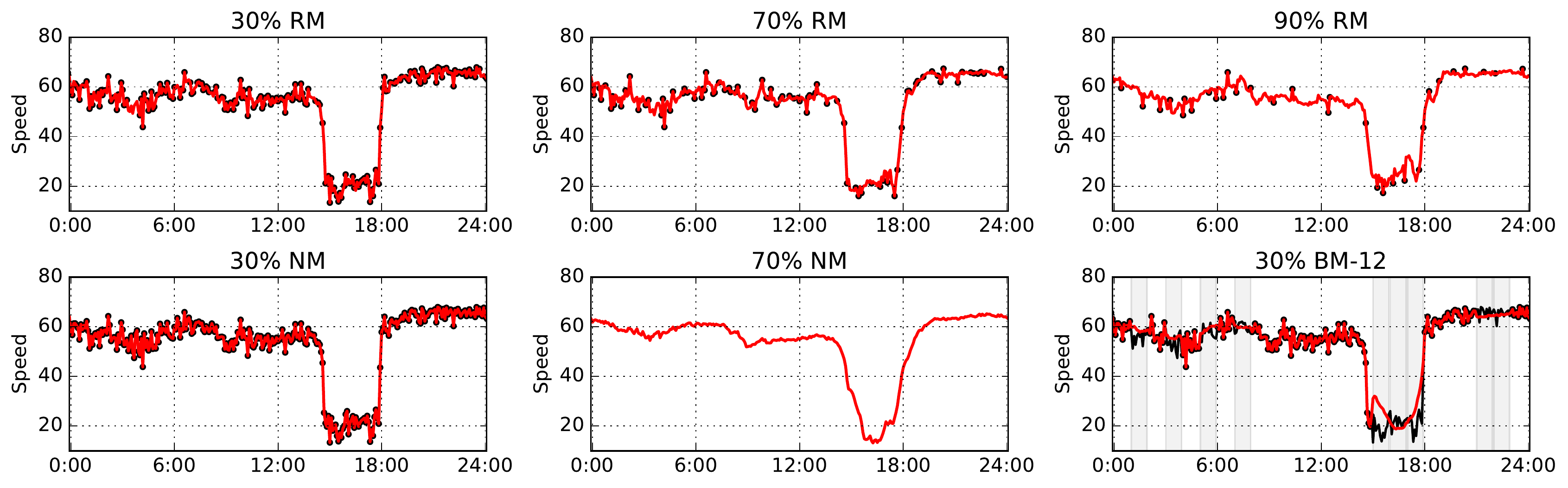}
\caption{Imputed values by LATC for Seattle freeway traffic speed data. This example corresponds to  detector \#3 and the 7th day of the data set.}
\label{seattle-imputation-plot}
\end{figure*}

\begin{figure*}[!ht]
\centering
\includegraphics[width=0.9\textwidth]{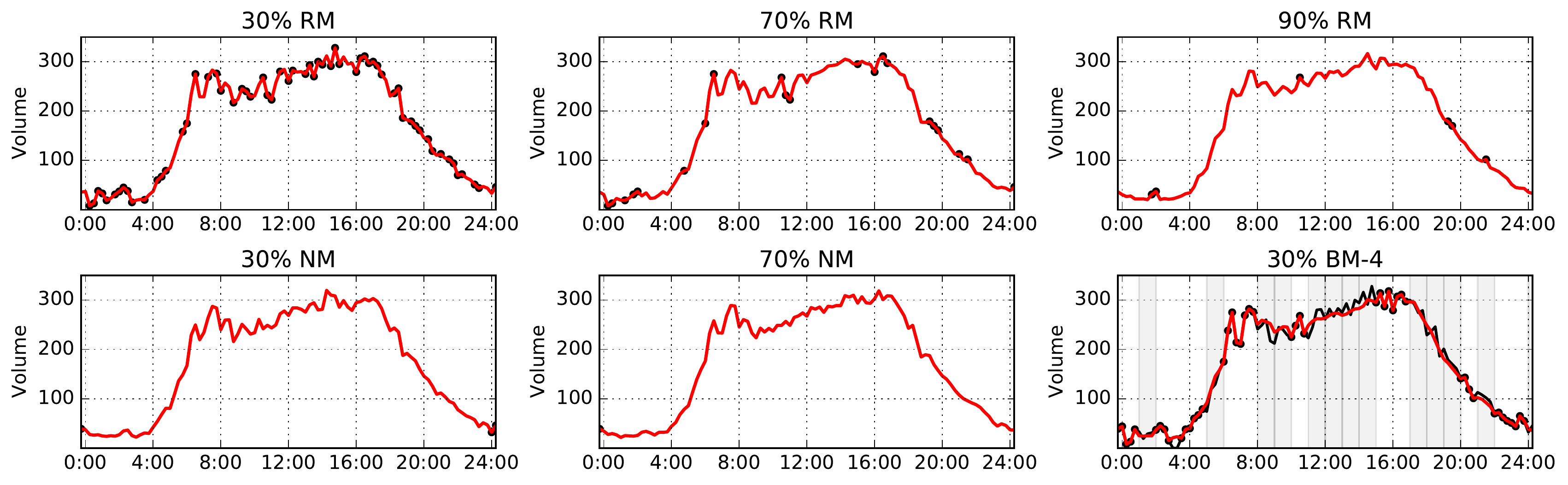}
\caption{Imputed values by LATC for Portland traffic volume data. This example corresponds to  detector  \#3 and the 8th day of the data set.}
\label{portland-imputation-plot}
\end{figure*}

\section{Conclusion}\label{conclusion}

Spatiotemporal traffic data imputation is of great significance in data-driven intelligent transportation systems. Fortunately, for analyzing and modeling traffic data, there are some fundamental features such as low-rank properties and temporal dynamics that can be taken into account. In this work, the proposed LATC model builds both low-rank structure (i.e., truncated nuclear norm) and time series autoregressive process on certain data representations. By doing so, numerical experiments on some real-world traffic data sets show the advantages of LATC over other low-rank models.
In addition to the imputation capability of LATC, LATC can also be applied to spatiotemporal traffic forecasting in the presence of missing values.

\appendices
\section{Supplementary theorem}\label{supp_theorem}

\begin{theorem}\label{theorem_ar}
Suppose $\boldsymbol{\Phi}_{0}\in\mathbb{R}^{(T-h_d)\times T}$, $\boldsymbol{\Phi}\in\mathbb{R}^{(T-h_d)\times (dT)}$, and autoregressive coefficient $\boldsymbol{A}\in\mathbb{R}^{M\times d}$ as defined in Lemma~\ref{lemma_ar}, then an optimal solution to the problem
\begin{equation*}\label{minimize_z_variable}
    \min_{\boldsymbol{Z}}~\frac{1}{2}\left\|\boldsymbol{\Phi}_{0}\boldsymbol{Z}^\top-\boldsymbol{\Phi}(\boldsymbol{A}^\top\odot\boldsymbol{Z}^\top)\right\|_{F}^{2}+\frac{\alpha}{2}\left\|\boldsymbol{Z}-\boldsymbol{X}\right\|_{F}^{2},
\end{equation*}
is given by
\begin{equation*}\label{solution_to_z_matrix}
    \operatorname{vec}(\boldsymbol{Z}^\top):=\alpha[(\boldsymbol{B}-\boldsymbol{C})^\top(\boldsymbol{B}-\boldsymbol{C})+\alpha\boldsymbol{I}_{MT}]^{-1}\cdot\operatorname{vec}(\boldsymbol{X}^\top),
\end{equation*}
where $\boldsymbol{B}=(\boldsymbol{I}_{M}\otimes\boldsymbol{\Phi}_0)$ and $\boldsymbol{C}=(\boldsymbol{I}_{M}\otimes \boldsymbol{\Phi})[(\boldsymbol{I}_{M}\odot\boldsymbol{A}^\top)\otimes\boldsymbol{I}_{T}]$. $\otimes$ denotes the Kronecker product.

\end{theorem}

\begin{proof}

In this case, we can use vectorization:
\begin{equation*}
\begin{aligned}
    \operatorname{vec}(\boldsymbol{\Phi}_0\boldsymbol{Z}^\top)=&(\boldsymbol{I}_{M}\otimes\boldsymbol{\Phi}_0)\cdot\operatorname{vec}(\boldsymbol{Z}^\top), \\
    \operatorname{vec}(\boldsymbol{\Phi}(\boldsymbol{A}^\top\odot\boldsymbol{Z}^\top))=&(\boldsymbol{I}_{M}\otimes \boldsymbol{\Phi})\cdot\operatorname{vec}(\boldsymbol{A}^\top\odot\boldsymbol{Z}^\top) \\
    =&(\boldsymbol{I}_{M}\otimes \boldsymbol{\Phi})[(\boldsymbol{I}_{M}\odot\boldsymbol{A}^\top)\otimes\boldsymbol{I}_{T}]\cdot\operatorname{vec}(\boldsymbol{Z}^\top), \\
\end{aligned}
\end{equation*}
where $\operatorname{vec}(\cdot)$ denotes the vectorization operator for any given matrix. Denote by $f$ the objective of problem \eqref{minimize_z_variable}:
\begin{equation*}
\begin{aligned}
    f=&\frac{1}{2}\|(\boldsymbol{B}-\boldsymbol{C})\cdot\operatorname{vec}(\boldsymbol{Z}^\top)\|_{2}^{2} +\frac{\alpha}{2}\|\operatorname{vec}(\boldsymbol{Z}^\top)-\operatorname{vec}(\boldsymbol{X}^\top)\|_{2}^{2}. \\
\end{aligned}
\end{equation*}
By letting %$\frac{df}{d\operatorname{vec}(\boldsymbol{Z}^\top)}=\boldsymbol{0}$, we have
\begin{equation*}
\begin{aligned}
    \frac{df}{d\operatorname{vec}(\boldsymbol{Z}^\top)}=&(\boldsymbol{B}-\boldsymbol{C})^\top(\boldsymbol{B}-\boldsymbol{C})\operatorname{vec}(\boldsymbol{Z}^\top) \\
    &+\alpha[\operatorname{vec}(\boldsymbol{Z}^\top)-\operatorname{vec}(\boldsymbol{X}^\top)]=\boldsymbol{0}, \\
\end{aligned}
\end{equation*}
we have
\begin{equation*}
\begin{aligned}
    \operatorname{vec}(\boldsymbol{Z}^\top)=&\alpha[(\boldsymbol{B}-\boldsymbol{C})^\top(\boldsymbol{B}-\boldsymbol{C})+\alpha\boldsymbol{I}_{MT}]^{-1}\cdot\operatorname{vec}(\boldsymbol{X}^\top). \\
\end{aligned}
\end{equation*}

\end{proof}

\section*{Acknowledgement}

This research is supported by the Natural Sciences and Engineering Research Council (NSERC) of Canada, the Fonds de recherche du Quebec – Nature et technologies (FRQNT), and the Canada Foundation for Innovation (CFI). X. Chen and M. Lei would  like  to  thank  the  Institute  for  Data  Valorisation (IVADO) for providing the PhD Excellence Scholarship to support this study.

% if have a single appendix:
%\appendix[Proof of the Zonklar Equations]
% or
%\appendix  % for no appendix heading
% do not use \section anymore after \appendix, only \section*
% is possibly needed

% use appendices with more than one appendix
% then use \section to start each appendix
% you must declare a \section before using any
% \subsection or using \label (\appendices by itself
% starts a section numbered zero.)
%

% \appendices
% \section{Proof of the First Zonklar Equation}
% Appendix one text goes here.

% % you can choose not to have a title for an appendix
% % if you want by leaving the argument blank
% \section{}
% Appendix two text goes here.

% % use section* for acknowledgment
% \section*{Acknowledgment}

% The authors would like to thank...

% Can use something like this to put references on a page
% by themselves when using endfloat and the captionsoff option.
\ifCLASSOPTIONcaptionsoff
  \newpage
\fi

% trigger a \newpage just before the given reference
% number - used to balance the columns on the last page
% adjust value as needed - may need to be readjusted if
% the document is modified later
%\IEEEtriggeratref{8}
% The "triggered" command can be changed if desired:
%\IEEEtriggercmd{\enlargethispage{-5in}}

% references section

% can use a bibliography generated by BibTeX as a .bbl file
% BibTeX documentation can be easily obtained at:
% http://mirror.ctan.org/biblio/bibtex/contrib/doc/
% The IEEEtran BibTeX style support page is at:
% http://www.michaelshell.org/tex/ieeetran/bibtex/
%\bibliographystyle{IEEEtran}
% argument is your BibTeX string definitions and bibliography database(s)
%\bibliography{IEEEabrv,../bib/paper}
%
% <OR> manually copy in the resultant .bbl file
% set second argument of \begin to the number of references
% (used to reserve space for the reference number labels box)
% \begin{thebibliography}{1}

% \bibitem{IEEEhowto:kopka}
% H.~Kopka and P.~W. Daly, \emph{A Guide to \LaTeX}, 3rd~ed.\hskip 1em plus
%   0.5em minus 0.4em\relax Harlow, England: Addison-Wesley, 1999.

% \end{thebibliography}

\bibliographystyle{IEEEtran}
\bibliography{tensor}

% biography section
%
% If you have an EPS/PDF photo (graphicx package needed) extra braces are
% needed around the contents of the optional argument to biography to prevent
% the LaTeX parser from getting confused when it sees the complicated
% \includegraphics command within an optional argument. (You could create
% your own custom macro containing the \includegraphics command to make things
% simpler here.)
%\begin{IEEEbiography}[{\includegraphics[width=1in,height=1.25in,clip,keepaspectratio]{mshell}}]{Michael Shell}
% or if you just want to reserve a space for a photo:

% \begin{IEEEbiography}{Xinyu Chen}
% Biography text here.
% \end{IEEEbiography}

% if you will not have a photo at all:
\begin{IEEEbiography}[{\includegraphics[width=1in,height=1.25in,clip,keepaspectratio]{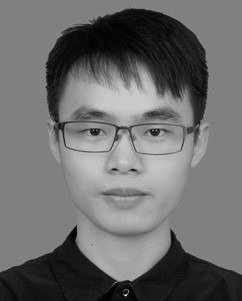}}]{Xinyu Chen} received the B.S. degree in Traffic Engineering from Guangzhou University, Guangzhou, China, in 2016, and M.S. degree in Transportation Information Engineering \& Control from Sun Yat-Sen University, Guangzhou, China, in 2019. He is currently a PhD student with the Civil, Geological and  Mining Engineering Department at Polytechnique Montreal, Montreal, QC, Canada. His current research centers on machine learning, spatiotemporal data modeling, and intelligent transportation systems.
\end{IEEEbiography}

\begin{IEEEbiography}[{\includegraphics[width=1in,height=1.1in]{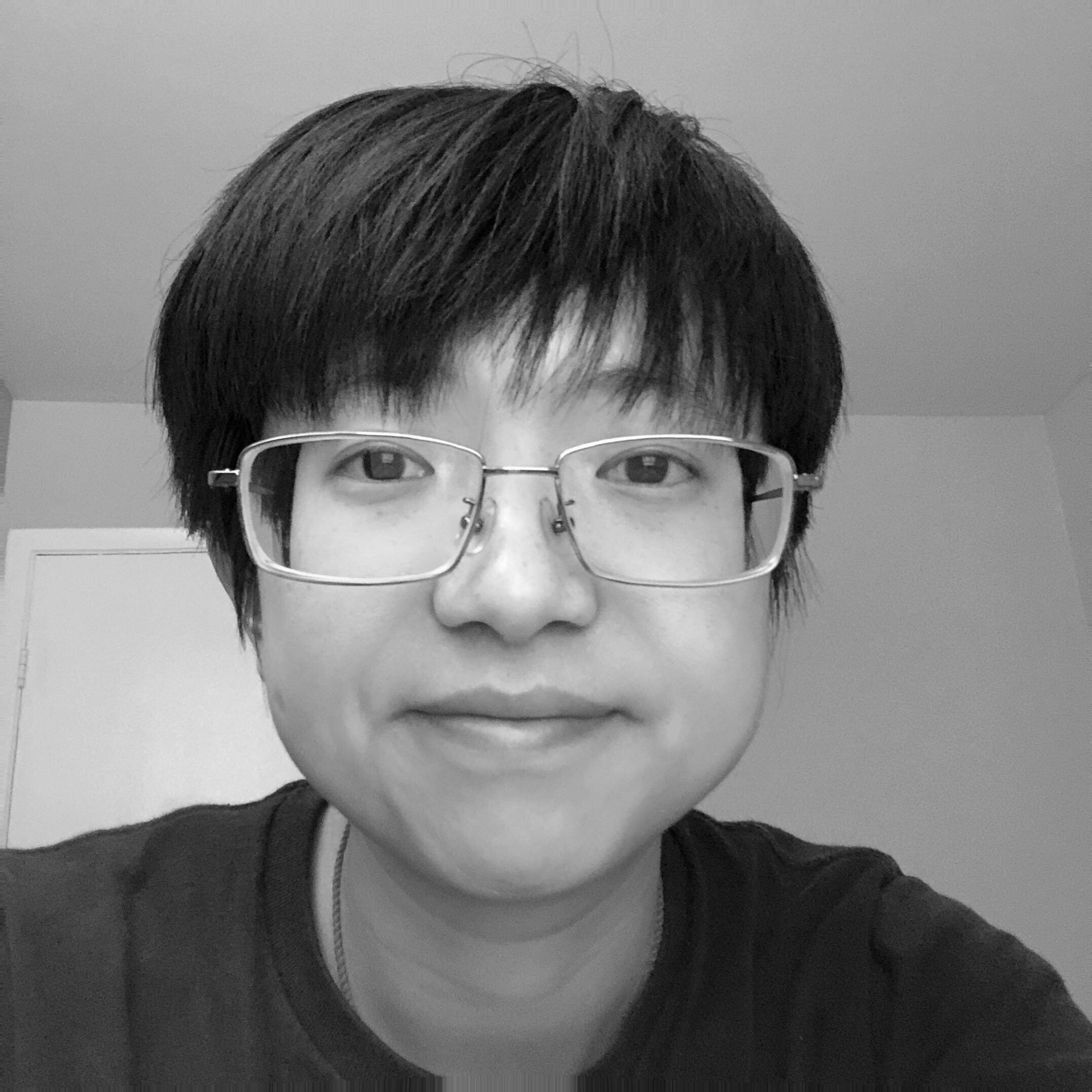}}]{Mengying Lei} received the B.S. degree in automation from Huazhong Agricultural University, in 2016, and the M.S. degree from the school of automation science and electrical engineering, Beihang University, Beijing, China, in 2019. She is now a Ph.D. student with the Department of Civil Engineering at McGill University, Montreal, Quebec, Canada. Her research currently focuses on spatiotemporal data modelling and intelligent transportation systems.
\end{IEEEbiography}

\begin{IEEEbiography}[{\includegraphics[width=1in]{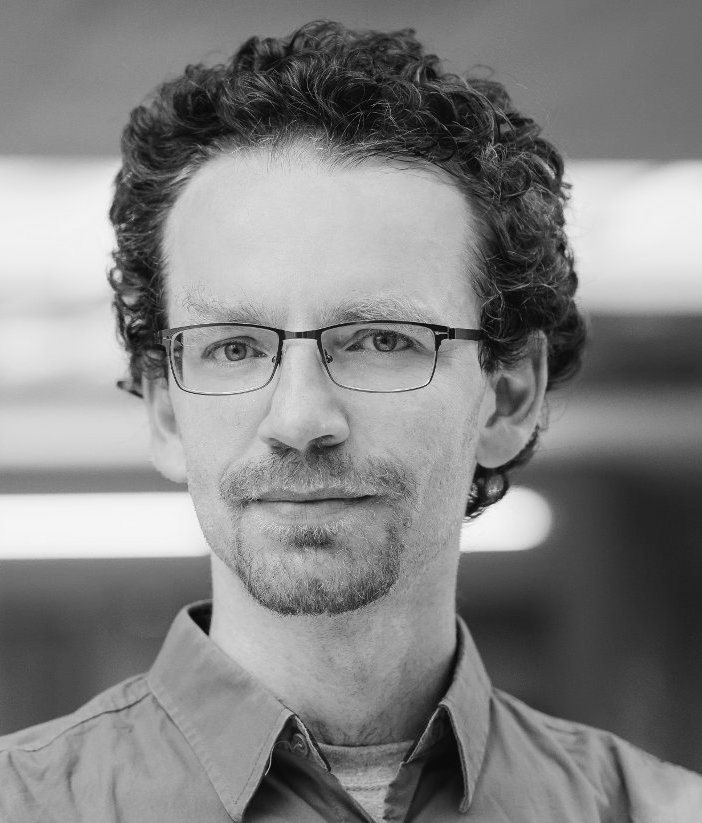}}]{Nicolas Saunier} received an engineering degree and a Doctorate (Ph.D.) in computer science from Telecom ParisTech, Paris, France, respectively in 2001 and 2005. He is currently a Full Professor with the Civil, Geological and Mining Engineering Department at Polytechnique Montreal, Montreal, QC, Canada. His research interests include intelligent transportation, road safety, and data science for transportation.
\end{IEEEbiography}

\begin{IEEEbiography}[{\includegraphics[width=1in,height=1.25in,clip,keepaspectratio]{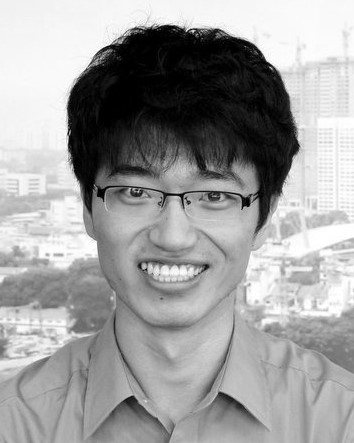}}]{Lijun Sun} received the B.S. degree in Civil Engineering from Tsinghua University, Beijing, China, in 2011, and Ph.D. degree in Civil Engineering (Transportation) from National University of Singapore in 2015. He is currently an Assistant Professor with the Department of Civil Engineering at McGill University, Montreal, QC, Canada. His research centers on intelligent transportation systems, machine learning, spatiotemporal modeling, travel behavior, and agent-based simulation. He is a member of the IEEE.
\end{IEEEbiography}

% You can push biographies down or up by placing
% a \vfill before or after them. The appropriate
% use of \vfill depends on what kind of text is
% on the last page and whether or not the columns
% are being equalized.

%\vfill

% Can be used to pull up biographies so that the bottom of the last one
% is flush with the other column.
%\enlargethispage{-5in}

% that's all folks
\end{document}